\documentclass{article}

\usepackage{times}
\usepackage{graphicx} % more modern
\usepackage{subfigure} 
\usepackage[usenames, dvipsnames]{color}

\usepackage{chngcntr}
\counterwithout{figure}{section}

\usepackage{natbib}

\usepackage{algorithm}
\usepackage{algorithmic}

\usepackage{amsmath}
\usepackage{amsthm}
\usepackage{amssymb}
\usepackage{url}
\usepackage{wrapfig}
\usepackage{mathtools}
\usepackage{float}

\newcommand\eqdef{\stackrel{\mathclap{\tiny\mbox{def}}}{=}}

\DeclareMathOperator*{\argmax}{argmax}
\DeclareMathOperator*{\argmin}{argmin}
\theoremstyle{plain}

\theoremstyle{plain}

\ifx\proof\undefined
\newenvironment{proof}[1][\protect\proofname]{\par
\normalfont\topsep6\p@\@plus6\p@\relax
\trivlist
\itemindent\parindent
\item[\hskip\labelsep\scshape #1]\ignorespaces
}{%
\endtrivlist\@endpefalse
}
\providecommand{\proofname}{Proof}
\fi

\newtheorem{theorem}{Theorem}
\newtheorem{definition}{Definition}
\newtheorem{example}{Example}

\newtheorem{lemma}{Lemma}
\newtheorem{remark}{Remark}

\providecommand{\propositionname}{Proposition}

\usepackage{hyperref}

\usepackage[accepted]{icml2017}
\usepackage{multicol,lipsum}
\usepackage[dvipsnames]{xcolor}

\begin{document} 

\twocolumn[
\icmltitle{Contrastive Principal Component Analysis}
%How about 'cPCA: Learning Patterns by Contrast'
\icmlsetsymbol{equal}{*}
 
\begin{icmlauthorlist}
\icmlauthor{Abubakar Abid}{equal,stanford}
\icmlauthor{Martin J. Zhang}{equal,stanford}
\icmlauthor{Vivek K. Bagaria}{stanford}
\icmlauthor{James Zou}{stanford}
\end{icmlauthorlist}

\icmlaffiliation{stanford}{Stanford University, CA, USA}

\icmlcorrespondingauthor{James Zou}{jamesz@stanford.edu}

\icmlkeywords{PCA, unsupervised, contrastive}

\vskip 0.3in

]
\printAffiliationsAndNotice{\icmlEqualContribution} 

\begin{abstract} 
We present a new technique called contrastive principal component analysis (cPCA) that is designed to discover low-dimensional structure that is unique to a dataset, or enriched in one dataset relative to other data. The technique is a generalization of standard PCA, for the setting where multiple datasets are available -- e.g. a treatment and a control group, or a mixed versus a homogeneous population -- and the goal is to explore patterns that are specific to one of the datasets. We conduct a wide variety of experiments in which cPCA identifies important dataset-specific patterns that are missed by PCA, demonstrating that it is useful for many applications: subgroup discovery, visualizing trends, feature selection, denoising, and data-dependent standardization. We provide geometrical interpretations of cPCA and show that it satisfies desirable theoretical guarantees.
We also extend cPCA to nonlinear settings in the form of kernel cPCA. We have released our code as a python package\textsuperscript{$\dagger$} and documentation is on Github\textsuperscript{$\ddagger$}.
\end{abstract} 

\section{Introduction}
\label{section:intro}
The principal component analysis (PCA) is one of the most widely-used methods for data exploration and visualization \cite{hotelling1933analysis}. PCA projects the data onto low dimensions and is especially powerful as an approach to visualize patterns, such as clusters and clines, in a dataset \citep{jolliffe2002principal}. In this paper, we extend PCA to the setting where we have multiple datasets and are interested in discovering patterns that are specific to, or enriched in, one dataset relative to another. % In this paper, we extend PCA to the setting where we have multiple datasets and are interested in discovering patterns that are specific to, or enriched in, one dataset relative to others. 
We illustrate why this is useful via two examples.

\vspace{3mm}

\textbf{Demographically-Diverse Cancer Patients.} 
Suppose we have gene-expression measurements from individuals of different ethnicities and sexes. This data includes gene-expression levels of cancer patients $\{X_{i}\}$, which we are interested in analyzing. We also have control data, which corresponds to the gene-expression levels of healthy patients $\{Y_{i}\}$ from a similar demographic background. Our goal is to find trends and variations within cancer patients  (e.g. to identify molecular subtypes of cancer).

If we directly apply PCA to $\{X_{i}\}$, however, the top principal components may correspond to the demographic variations of the individuals instead of the subtypes of cancer because the genetic variations due to the former are likely to be larger than that of the latter \citep{garte1998role}. As we show, we can overcome this problem by noting that the healthy patients also contain the variation associated with demographic differences, but not the variation corresponding to subtypes of cancer. Thus, we can search for components in which $\{X_{i}\}$ has high variance but $\{Y_{i}\}$ has \textit{low} variance.

\textbf{Handwritten Digits on Complex Backgrounds.}
As another example, consider a dataset $\{X_{i}\}$ that consists of handwritten digits on a complex background, such as different images of grass (see Fig. \ref{fig:mnist_on_grass}a). A typical unsupervised learning task may be to cluster the data according to the digits in the image. However, if we perform standard PCA on these images, we find that the top principal components do not represent features related to the handwritten digits, but reflect the dominant variation in features related to the image background (Fig. \ref{fig:mnist_on_grass}b).

We will show that it is possible to correct for this by using a reference dataset $\{Y_{i}\}$ that consists solely of images of the grass (not necessarily the same images used in $\{X_{i}\}$ but having similar covariance between features (see Fig. \ref{fig:mnist_on_grass}c), and looking for the subspace of \textit{higher} variance in $\{X_{i}\}$ compared to $\{Y_{i}\}$. By projecting onto this subspace, we can actually visually separate the images based on the value of the handwritten digit, as shown in Fig. \ref{fig:mnist_on_grass}d. 

\begin{figure*}[]
\label{fig:mnist_on_grass}
\begin{center}
\subfigure[]{\includegraphics[width=0.18\textwidth]{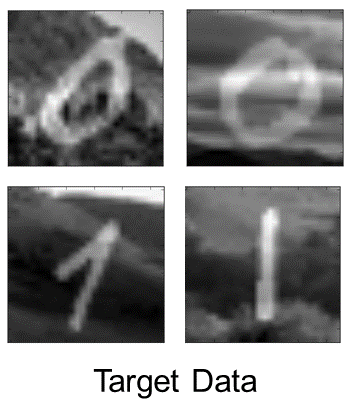}} \hspace{2mm}
\subfigure[]{\includegraphics[width=0.23\textwidth]{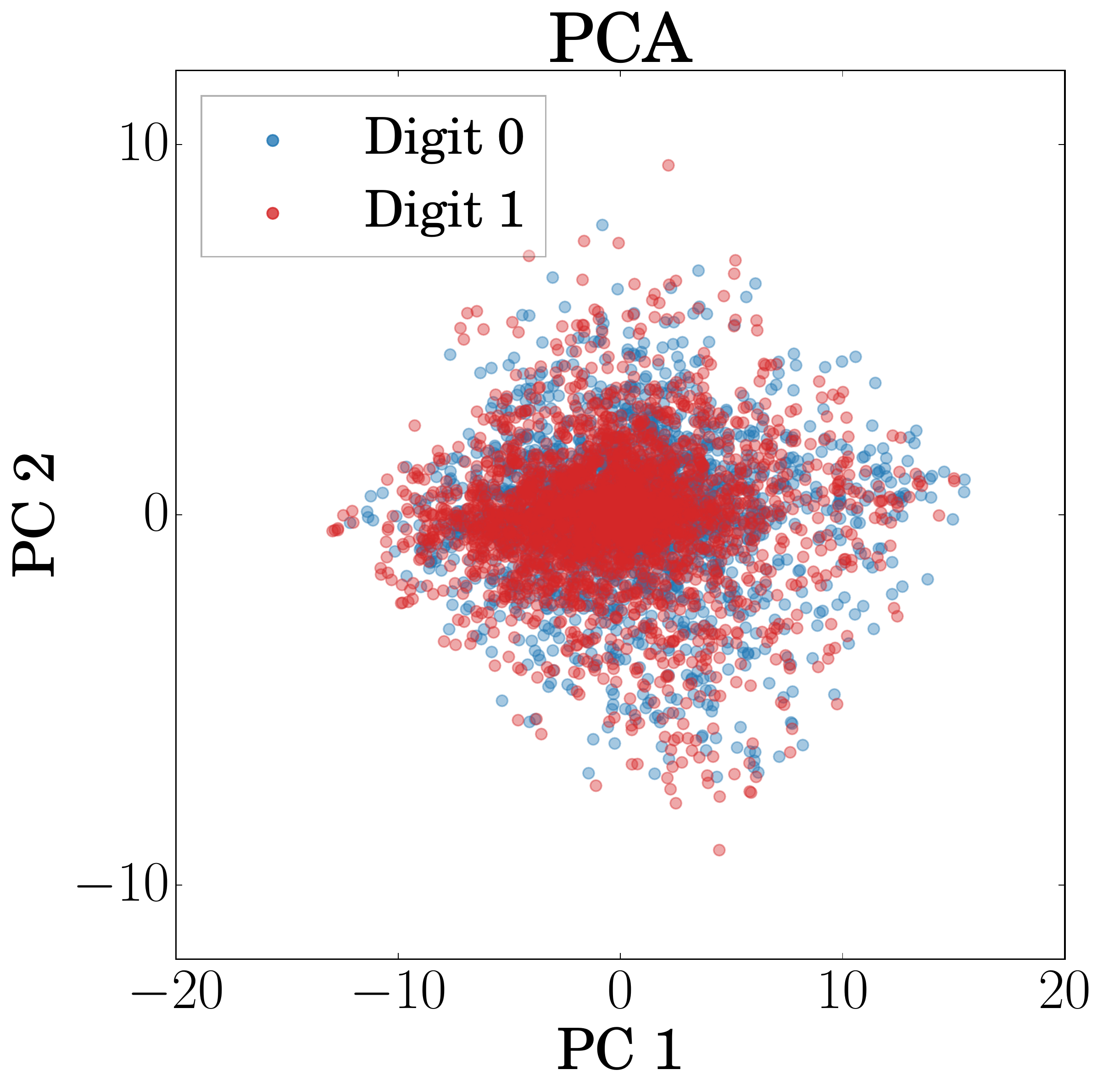}} \hspace{2mm}
\subfigure[]{\includegraphics[width=0.18\textwidth]{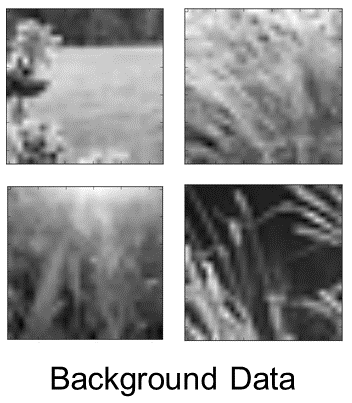}} \hspace{2mm}
\subfigure[]{\includegraphics[width=0.22\textwidth]{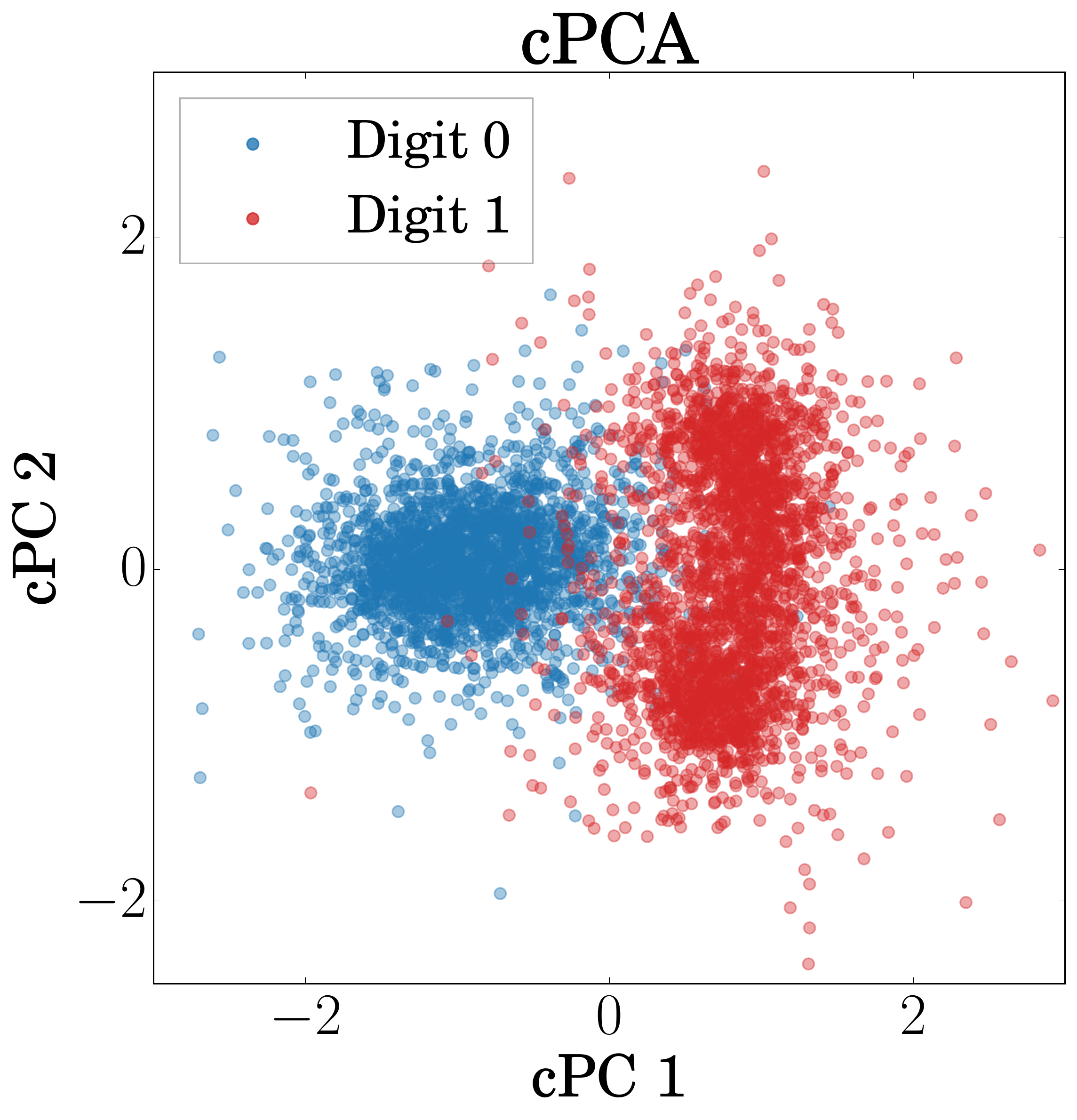}}
\end{center}
\caption{\textbf{Contrastive PCA on Synthetic Images.} (a) We create a target dataset of 5,000 synthetic images by randomly superimposing images of handwritten digits 0 and 1 from the MNIST dataset \citep{lecun1998gradient} on top of images of grass taken from ImageNet dataset \citep{russakovsky2015imagenet} belonging to the synset \textit{grass}. The images of grass are converted to grayscale, resized to be 100x100, and then randomly cropped to be the same size as the MNIST digits, 28x28. (b) Here, we plot the result of embedding the synthetic images onto their first two principal components using standard PCA. We see that the lower-dimensional embeddings of the images with 0s and images with 1s are hard to distinguish. (c) A background dataset is then introduced consisting solely of images of grass belonging to the same synset, but we use images that are different than those used to create the target dataset. (d) Using cPCA on the target and background datasets, (with a value of the contrast parameter $\alpha$ [see section \ref{section:methods}] set to $2.0$), two clusters emerge in the lower-dimensional representation of the target dataset, one consisting of images with the digit 0 and the other of images with the digit 1.}
\end{figure*}

In the above examples, $\{Y_{i}\}$ can be viewed as the \textit{background dataset} that has the universal but uninteresting features, and $\{X_{i}\}$ can be viewed as the \textit{target dataset} that carries not only the universal features but some additional interesting features as well. Then the goal can be informally stated as finding directions in which the target data varies significantly, but the background data does not. 

\vspace{2mm}

Despite this being a simple-to-state and ubiquitous problem, we currently lack a principled framework to identify such contrastive dimensions in the literature. In this work, we develop contrastive PCA, referred to as cPCA, which takes as input datasets $X$ and $Y$ and efficiently identifies lower-dimensional subspaces that capture structure specific to $X$. These directions correspond to features carried uniquely by the target dataset, and are hence are more likely to lead to meaningful discovery of the additional structure of the target data compared to the background. Indeed, through multiple experiments, we demonstrate that cPCA can be a powerful tool for data exploration, lending itself to a variety of unsupervised learning applications.
These experiments are supported by certain theoretical guarantees that we prove for cPCA.
Moreover, we extend our algorithm to nonlinear settings in the form of kernel cPCA.

\subsection{Related works}
PCA is a linear dimensionality-reduction technique that is most commonly used to visualize and explore a single dataset. There is a large number of related visualization methods; for example, t-SNE \citep{maaten2008visualizing} and multi-dimensional scaling (MDS) \citep{cox2008multidimensional} allow for nonlinear data projections and may better capture nonlinear patterns than PCA. Yet, all of these methods are designed to explore one dataset at a time. When the analyst has multiple datasets, which is often the case, then the current state-of-practice is to perform PCA (or t-SNE, MDS, etc.) on each dataset separately, and then manually compare the various projections to explore if there are interesting similarities and differences across data. 
% If there are strong structures that are common across the datasets, then these common structures would dominate the PCA of each of the data, making it challenging to visualize the more subtle patterns that are specific to a particular data. 
cPCA is designed to fill in this gap in data exploration and visualization by automatically identifying the projections that exhibit the most interesting differences across datasets. 
 
We note that the primary usages of cPCA are exactly in settings where PCA is popularly deployed: efficiently reduce dimensions to enable visualization and exploratory data analysis. cPCA is a tool for unsupervised learning. This separates cPCA from a large class of supervised learning methods whose primary goal is to classify or discriminate between the various datasets. The standard supervised learning approach is to learn a classifier to predict whether a given data point comes from the target or the background set. cPCA does not try to classify individual datum; instead it seeks to visualize patterns that are specific to the target. The two-group differential statistics methods -- Fisher's discriminant analysis, two-sample t-test, Wilcoxon signed-rank test, Mann-Whitney U test -- aim to identify features that differ in their means (or other statistic) between the target and background groups \citep{du2010choosing}. While these differential features and statistics capture significant differences between the two datasets, they do not try to discover patterns in the target data itself, unlike cPCA. 

In Section \ref{section:applications}, we will demonstrate several example applications of cPCA. In a specific application domain, there may be specialized tools in that domain with similar goals as cPCA. For example, in Section \ref{subsection:ancestry}, we will show how cPCA applied on genotype data visualizes geographical ancestry within Mexico. Exploring fine-grained clusterings of genetic ancestries is an important problem in population genetics, and researchers have recently developed an algorithm to specifically visualize such ancestry clusters \citep{moreno2014genetics}. While cPCA performs well here, the expert-crafted algorithm might perform even better for a specific dataset. However, the specialized algorithm requires substantial domain knowledge to design, is more computationally expensive, and can be challenging to use. The goal of cPCA is not to replace all these specialized state-of-the-art methods in each of their domains, but to provide a general method for exploring arbitrary datasets. Whenever PCA is used to identify patterns in related datasets, cPCA can be used with essentially the same computational efficiency.

\section{The cPCA Algorithm}
\label{section:methods}

Let $\{X_i\}_{i=1}^n$ and $\{Y_j\}_{j=1}^m$ be two datasets where $X_i, Y_i \in \mathbb{R}^d$. 
We are interested in finding patterns that are enriched in the $X_i$'s relative to the $Y_i$'s.
As before, we refer to $\{X_i\}$ as the target data and $\{Y_j\}$ as the background data.
Without loss of generality, we assume the data have been centered and 
use $C_X$ and $C_Y$ to denote their respective empirical covariance matrices.

Let $\mathbb{R}^d_{\text{unit}}$ be the set of vectors in $\mathbb{R}^d$ with unit norm. 
For any direction $\mathbf{v}\in \mathbb{R}^d_{\text{unit}}$, its corresponding variances in the target and in the background can be written as 
\begin{align*}
& \text{Target variance:}\quad\lambda_X(\mathbf{v}) \eqdef \mathbf{v}^T C_X \mathbf{v}.\\
&\text{Background variance:}\quad\lambda_Y(\mathbf{v}) \eqdef \mathbf{v}^T C_Y \mathbf{v}.
\end{align*}

The goal of cPCA is to identify directions $\mathbf{v}$ which account for large variances in the target and small variances in the background. 
Specifically, cPCA solves the following optimization problem:
\begin{align}\label{eq:opt_obj}
\text{argmax}_{\mathbf{v}\in\mathbb{R}^d_{\text{unit}}} \lambda_X(\mathbf{v}) - \alpha \lambda_Y(\mathbf{v}) 
\end{align}
where $\alpha \in [0, \infty]$ is a parameter discussed below. This optimization problem is equivalent to $ \text{argmax}_{\mathbf{v}\in \mathbb{R}^d_{\text{unit}}} \mathbf{v}^T (C_X - \alpha C_Y) \mathbf{v}$, which can be efficiently solved by conducting an eigenvalue decomposition on the matrix $C \eqdef (C_X - \alpha C_Y)$ and returning the eigenvectors corresponding to the leading eigenvalues. Analogously to PCA, we call the leading eigenvectors the \emph{contrastive principal components} (cPCs) and we return the subspace spanned by the first few (typically two) orthogonal cPCs, as outlined in Algorithm \ref{alg:cpca1}. For a suitable $\alpha$, projecting $\{X_i\}$ onto this subspace provides insight into structure specific to data.

The contrast parameter $\alpha$ represents the trade-off between maximizing the target variance and minimizing the background variance. 
When $\alpha=0$, cPCA selects the directions that only maximize the target variance, and hence reduces to PCA applied on $\{X_i\}$. 
As $\alpha$ increases, directions that reduce the background variance become more optimal and the contrastive principal components are driven towards the null space of the covariance matrix of $\{Y_i\}$.
When $\alpha=\infty$, any direction not in the null space of the background data receives a infinite penalty. In this case, cPCA is reduced to first projecting the target data on the null space of the background, and then performing PCA on the projected data. 
Therefore, each value of $\alpha$ yields a direction with a different trade-off between target and background variance. 

% \james{Say a bit more about what happens if $\alpha = \infty$. Does this case also reduce to some known method and how does it compare with something like Fisher's discriminant analysis?} 

Instead of choosing a single value of $\alpha$, in practice, we automatically select a few distinct values in such a way that the subspaces corresponding to each value of $\alpha$ lie far apart from one another, as characterized by the principal angle between the subspaces \citep{Miao1992}. This allows us to present the user with a few scatterplots, one for each selected value of $\alpha$ (that the user can quickly scan), which represent the behavior of cPCA for a wide range of $\alpha$ values, making the overall algorithm effectively hyperparameter-free. In some cases, each selected value of $\alpha$ reveals different structure within the target dataset. In Appendix \ref{supp:synthetic}, we show an example with synthetic data, where our algorithm automatically discovers various valid ways to subgroup data within the target.

The process of selecting values of $\alpha$ automatically is based on spectral clustering \citep{ng2002spectral} of an affinity matrix, where the affinity is the product of the cosine of the principal angles between the subspaces, as described in Algorithm \ref{alg:cpca2}. As it includes Algorithm \ref{alg:cpca1} as a subroutine, Algorithm \ref{alg:cpca2} is the complete algorithm for cPCA. We denote Algorithm \ref{alg:cpca2} as \texttt{cPCA} in this paper. \texttt{cPCA} selects from a list of potential values of $\alpha$ the few that yield the most representative subspaces for projecting the target data across the entire range of values of $\alpha$. We have found that beginning with $40$ values of $\alpha$ that are logarithmically spaced between 0.1 and 1000 yields good subspaces on a variety of datasets. Unless noted otherwise, these are the parameters used to perform the experiments in this paper. 

\begin{algorithm}[]
   \caption{Contrastive PCA For a Given $\alpha$}
   \label{alg:cpca1}
\begin{algorithmic}
   \STATE {\bfseries Inputs:} target and background data: $\{X_i \}_{i=1}^n$, $\{Y_i \}_{i=1}^m$; contrast parameter, $\alpha$; the \# of components, $k$
   \vspace{2mm}
   \STATE Centering the data $\{X_i \}_{i=1}^n$, $\{Y_i \}_{i=1}^m$.
   \STATE Calculate the empirical covariance matrices: 
   \vspace{-2mm}
   $$ C_X = \frac{1}{n} \sum_{i=1}^n X_i X_i^T, C_Y = \frac{1}{m} \sum_{i=1}^m Y_i Y_i^T $$
%    \FOR{each $\alpha$}
   \vspace{-2mm}
   \STATE Perform eigenvalue decomposition on 
   \vspace{-2mm}
   $$C = (C_X-\alpha C_Y)$$
   \STATE Compute the the subspace $V \in\mathbb{R}^k$ spanned by the top $k$ eigenvectors of $C$
%    \ENDFOR
   \STATE{\bfseries Return:} the subspace $V$ 
\end{algorithmic}
\end{algorithm}

	\begin{algorithm}[]
   \caption{Contrastive PCA with Auto Selection of $\alpha$}
   \label{alg:cpca2}
\begin{algorithmic}
   \STATE {\bfseries Inputs:} target and background data: $\{X_i \}_{i=1}^n$, $\{Y_i \}_{i=1}^m$; list of possible $\{\alpha_i\}$; the \# of components, $k$; $p$, the number of $\alpha$'s to present.
   \vspace{2 mm}
   \FOR{each $\alpha_i$}
   \STATE Compute the subspace $V_i$ using Algorithm \ref{alg:cpca1} with the contrast parameter set to $\alpha_i$.
   \ENDFOR
   \FOR{each pair $V_i,V_j$}
   \STATE Compute the principal angles $\theta_1 \ldots \theta_k$ between $V_i, V_j$ 
   \STATE Define the affinity $d(V_i,V_j) = \prod_{h=1}^{k}{\cos{\theta_h}}$ 
   \ENDFOR
   \STATE With $D_{ij} = d(V_i,V_j)$ as an affinity matrix between subspaces, do spectral clustering on $D$ to produce $p$ clusters
   \FOR{each cluster of subspaces $\{c_i\}_{i=1}^p$}
   \STATE Compute its \text{medoid}, $V^*_i$ the subspace defined as
   $$V_i^* \eqdef \arg \max_{V \in c_i} \sum_{V' \in c_i} d(V,V')$$
   \STATE Let $\alpha^*_{i}$ be the contrast parameter corresponding to $V^*_i$
   \ENDFOR
   \STATE{\bfseries Return:} $\alpha^*_{1} \cdots \alpha^*_{p}$ and the subspaces $V^*_{1} \cdots V^*_{p}$ 
\end{algorithmic}
\end{algorithm}

\subsection{Choosing the background dataset} 
The additional insight that cPCA reveals about a target dataset is based on characteristics of the background, leading us to ask: how do we choose a good contrastive background dataset? In general, the background should be chosen to have the structure that we want to \textit{remove} from the target data. This structure may correspond to variables that we are not interested in but may have significant variation in the target data. As examples, we list several possible background datasets and illustrate some of these examples through experiments in Section \ref{section:applications}.

\begin{itemize}
\itemsep0em 
\item A \textbf{control group} $\{Y_i\}$ contrasted with a diseased population $\{X_i\}$ because the control group contains similar population-level variation but not the subtle variation due to different subtypes of the disease (Fig. \ref{fig:mice_and_single_cell}).
\item A \textbf{homogeneous group} $\{Y_i\}$ contrasted with a mixed group $\{X_i\}$ because both have intra-population variation and measurement noise, but the former does not have inter-population variation (Fig. \ref{fig:single_cell}). 
\item A \textbf{`before-treatment'} dataset $\{Y_i\}$ contrasted with `after-treatment' data $\{X_i\}$ to remove measurement noise but preserve variation caused by treatment.
\item A \textbf{mixed group} $\{Y_i\}$ contrasted with another mixed group $\{X_i\}$ to remove general batch or mixing effects but not the variation due to the specific subclasses in the latter group (Fig. \ref{fig:ancestry}).
\item A set of \textbf{signal-free} recordings $\{Y_i\}$ or images that contain only noise, contrasted with measurements that consist of signal and noise $\{X_i\}$ (Figs. \ref{fig:mnist_on_grass} \& \ref{fig:mhealth}).
\end{itemize}

It is worth adding that the background data does \textit{not} need to have \textit{exactly} the same covariance structure as what we would like to remove from the target dataset. As an example, in the experiment shown in Fig. \ref{fig:mnist_on_grass}, it turns out that we do not need to use a background dataset that consists of images of grass. In fact, similar results are obtained even if instead of images of grass, images of bikes or the sky are used as the background dataset. As the structure of covariance matrices are similar enough, cPCA removes this background structure from the target. In the next section, we observe that this flexibility allows contrastive PCA to be useful in discovering patterns in a variety of real-world settings.

\section{Applications of Contrastive PCA}
 \label{section:applications}

As a general unsupervised learning technique, cPCA (like PCA) has a variety of uses in data exploration and visualization. Here, we show four applications where cPCA can provide additional insight into the structure of a dataset that is missed by standard PCA. The datasets that we employ in our experiments are summarized in Table \ref{table:experiments}.

In each of the examples that follow, we perform cPCA with three values of $\alpha$ that are automatically selected according to Algorithm \ref{alg:cpca2}. For simplicity, we show only the result of one of the values of $\alpha$ in the figures, and in a few cases, omit a small number of points to make the figures more illustrative. (see Appendix \ref{supp:full_results} for complete results).
\vspace{-3mm}
\begin{table}[h]
\caption{Characteristics of the datasets used in experiments: the target size ($n$), the background size ($m$), the dimensionality ($d$), and the number of subgroups within the target dataset ($sg$).}
% \abu{Consider replacing with a series of figures, including a 3D plot for 4 groups [mplot3d]}
\label{table:experiments}
\begin{center}
\begin{small}
\begin{sc}
\begin{tabular}{cccccr}
\hline
\abovespace\belowspace
Experiment (Fig \#)& $n$ & $m$ & $d$ & $sg$ \\
\hline
\abovespace
%0.19 +- 0.06
MNIST on Grass (\ref{fig:mnist_on_grass}) & 5,000 & 5,000 & 784 & 2 \\
Mice Protein Exp. (\ref{fig:mice_and_single_cell}) & 270 & 135 & 77 & 2 \\
Leuk. RNA-Seq 1\footnotemark[1] (\ref{fig:single_cell}) & 7,898 & 1,985 & 32,738 & 2 \\
Leuk. RNA-Seq 2\footnotemark[1] (\ref{fig:single_cell}) & 12,399 & 1,985 & 32,738 & 4 \\
Mexican Ancestry\footnotemark[1] (\ref{fig:ancestry}) & 241 & 507 & 3 $\cdot$ 10$^6$ & 5 \\
\belowspace
MHealth Sensors (\ref{fig:mhealth}) & 6,451 & 3,072 & 23 & 2 \\
\hline
\end{tabular}
\end{sc}
\end{small}
\end{center}
\vskip -0.1in
\end{table}
\vspace{-3mm}

\footnotetext[1]{As described in Sections \ref{subsection:subgroups} and \ref{subsection:ancestry}, these datasets were preprocessed to reduce dimensionality before cPCA was applied.}

\subsection{Discovering Subgroups in Complex Data}
\label{subsection:subgroups}

Researchers in the life sciences have noted that PCA is often ineffective at discovering subgroups within biological data, at least in part because ``dominant principal components \ldots correlate with artifacts" rather than with features that are of interest to the researcher \citep{ringner2008principal}. As described earlier, in many biomedical datasets, the principal components (PCs) reflect dominant but uninteresting variation that is present within a cohort, such as demographic variation, which overshadows the more subtle variation, e.g. markers associated with subtypes of cancer, that is of interest in a particular analysis. As a result, when data is projected onto the PCs, even if clusters form, they reflect less relevant groupings.
% In the preceding example, clusters may form based on the sex and race of participants instead of the subtype of lung cancer.

How can cPCA can be used to detect the more significant subgroups? By using a background dataset to cancel out the uninteresting variation as the target, we can search for structure that is unique to the target dataset. Indeed, as we show in the following two examples, when we project onto the contrastive principal components (cPCs) instead of the PCs, we discover subgroups in complex data that correspond to important differences within our target population.

\textbf{Mice Protein Expression.} Our first target dataset, adapted from a public dataset \citep{higuera2015self}, consists of protein expression measurements of mice that have received shock therapy. We assume that unbeknownst to the analyst, some of the mice have developed Down Syndrome. We would like to see if we detect any significant differences within the shocked mice population in an unsupervised manner (the presence or absence of Down Syndrome being a key example!). In Fig. \ref{fig:mice_and_single_cell}a, we show the result of applying PCA to the target dataset: the transformed data does not reveal any significant clustering within the population of mice. The major sources of variation within mice are likely natural, such as sex or age.

Next, we apply cPCA using a background dataset that consists of protein expression measurements from set of mice that have not been exposed to shock therapy or any similar experimental conditions. They are control mice that likely have similar natural variation as the experimental mice, but without the differences that result from the shock therapy. With this dataset as a background, we are able to resolve two different groups in the transformed target data, one corresponding to mice that do not have Down Syndrome and one corresponding (mostly) to mice that have Down Syndrome, as illustrated in Fig. \ref{fig:mice_and_single_cell}b.

\begin{figure}[!tb]

\subfigure[]{\includegraphics[width=0.23\textwidth]{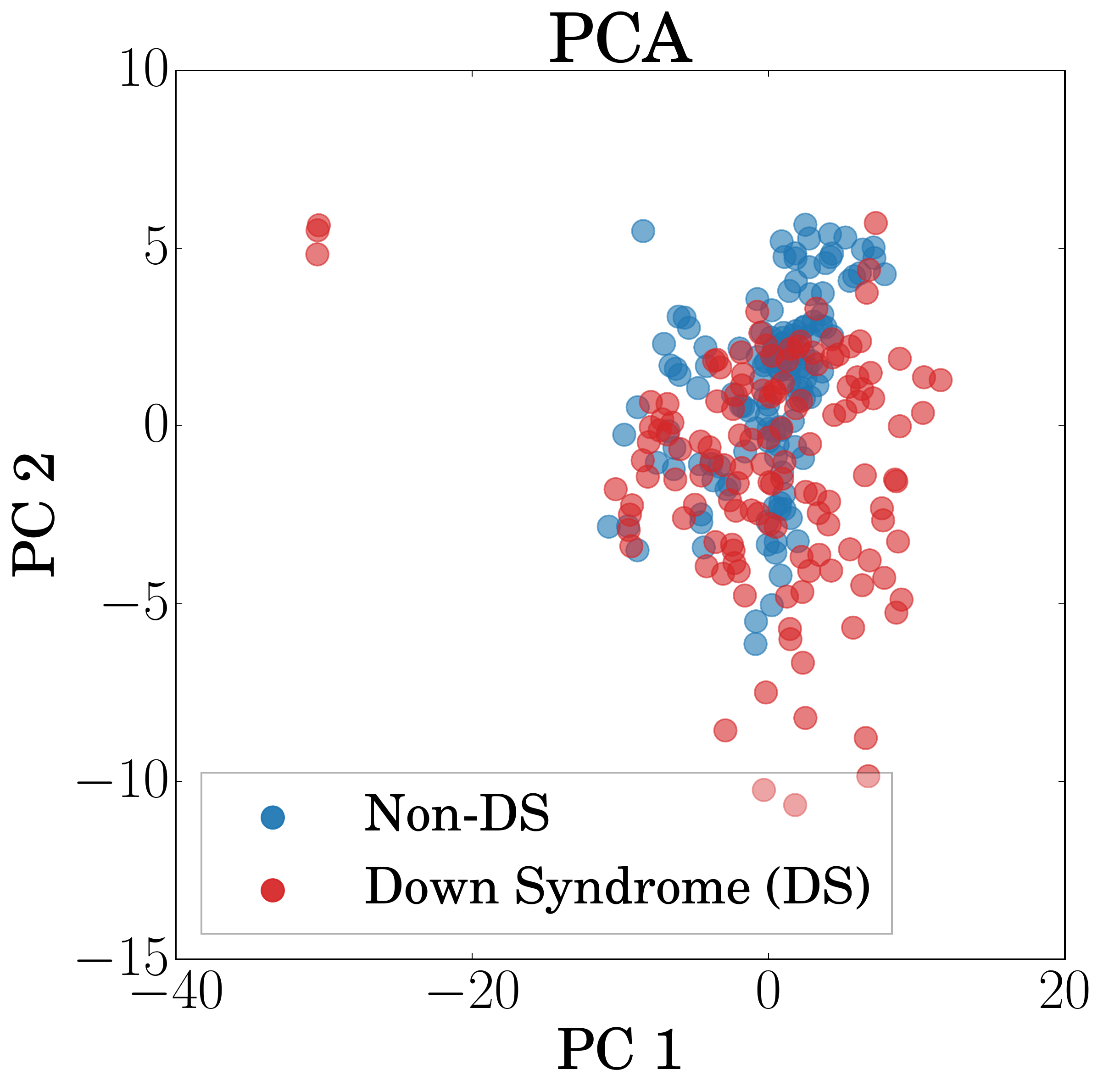}}
\subfigure[]{\includegraphics[width=0.224\textwidth]{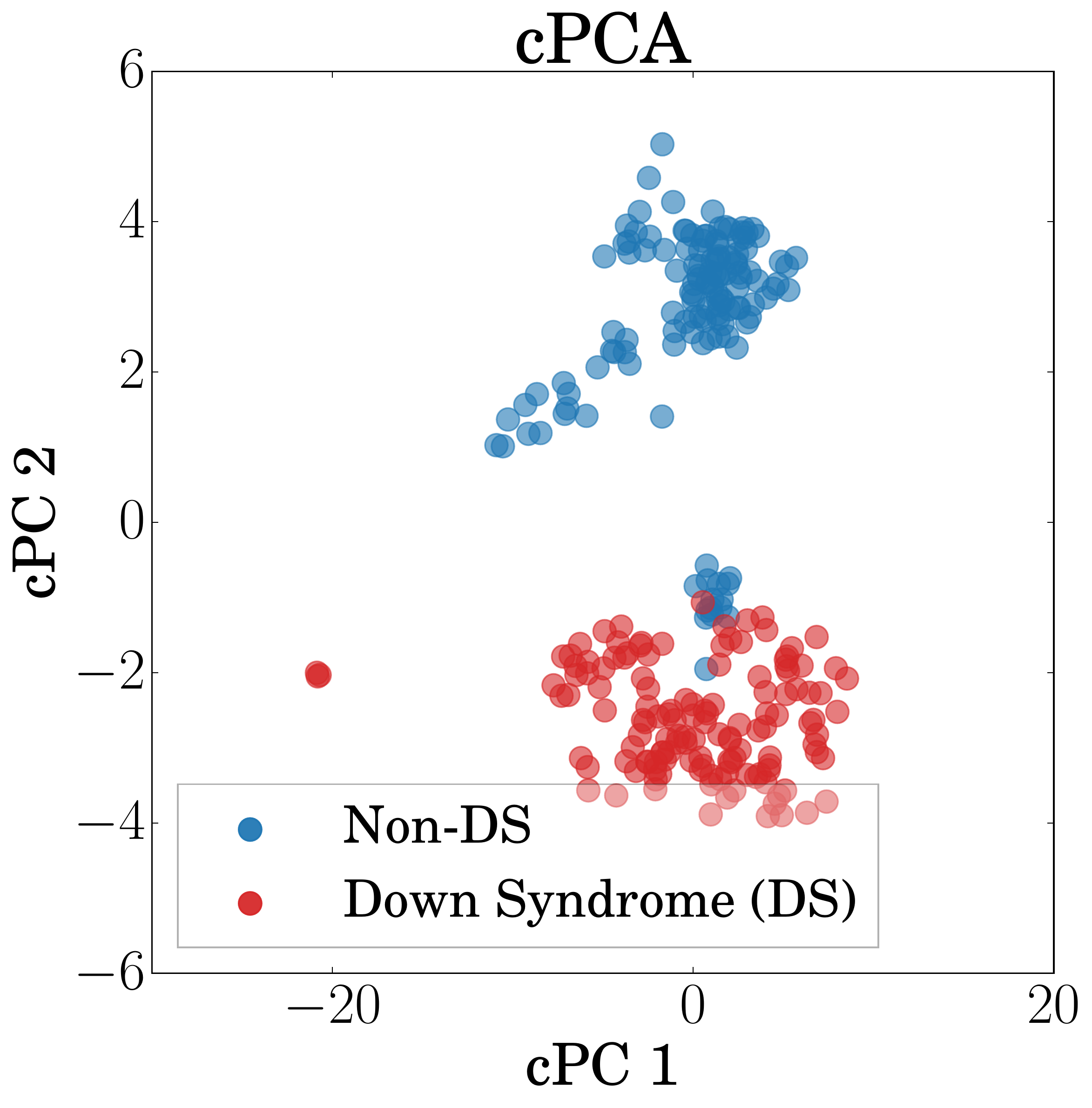}}

\caption{\textbf{Discovering Subgroups in Mice Protein Data.} (a) We use PCA to project a protein expression dataset of mice with and without Down Syndrome (DS) into its first two PCs. The lower-dimensional representation of protein expression measurements from mice with and without DS are seen to be distributed similarly. (b) We use cPCA to project the dataset onto its first two cPCs, discovering a lower-dimensional representation that clusters mice with and without DS separately.}
\label{fig:mice_and_single_cell}
\end{figure}

\begin{figure}[!tb]

\subfigure[]{\includegraphics[width=0.23\textwidth]{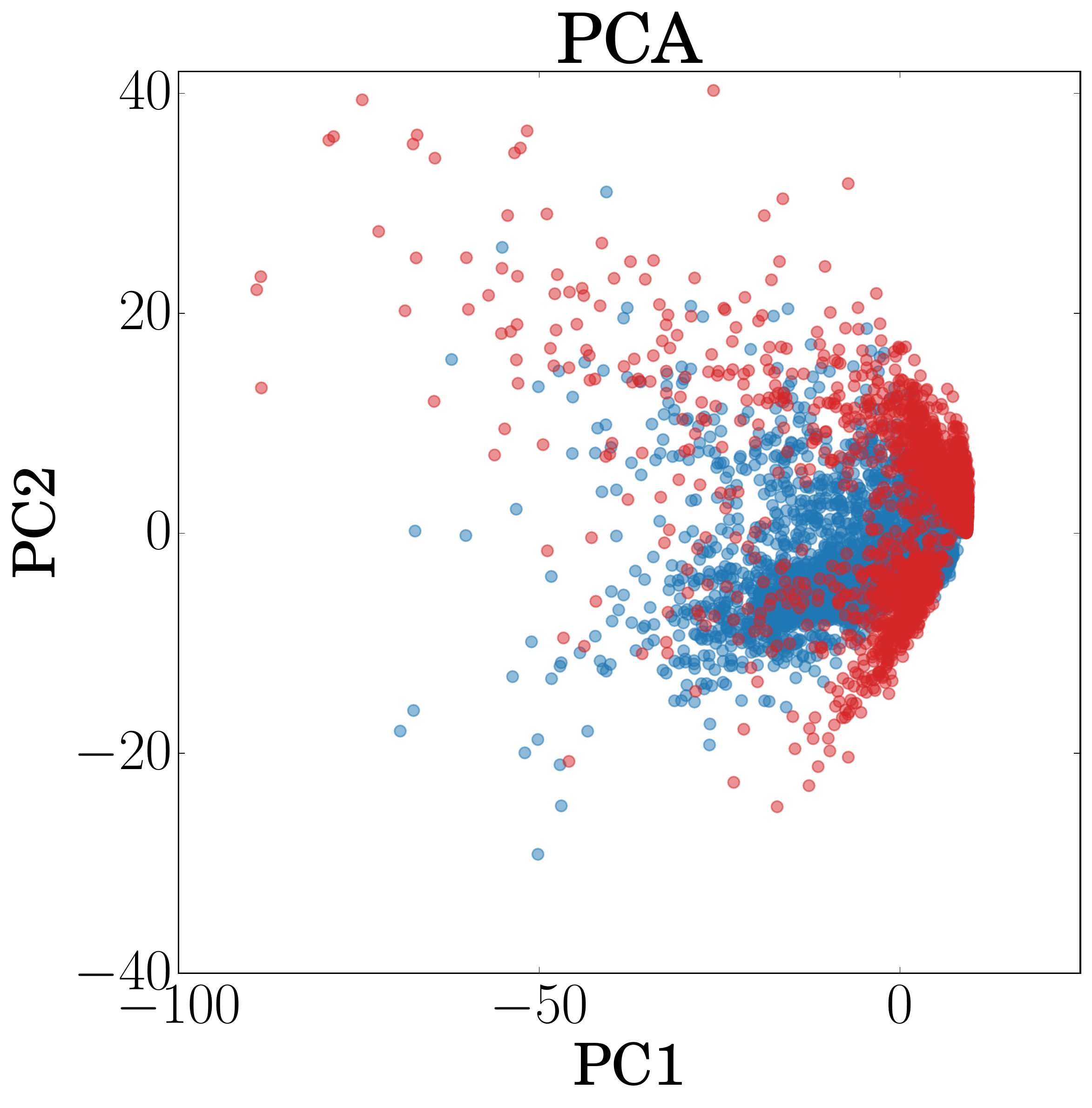}}
\subfigure[]{\includegraphics[width=0.224\textwidth]{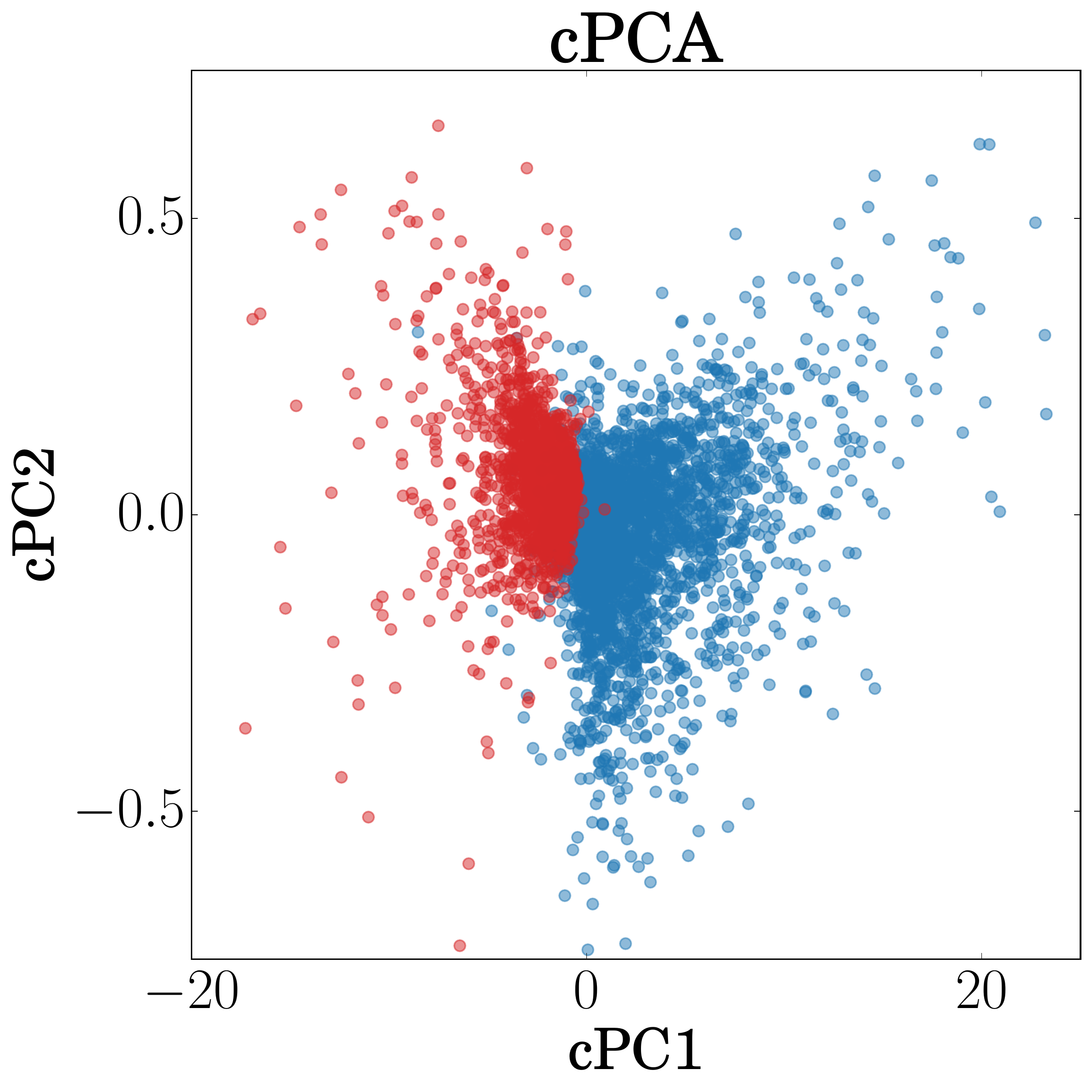}} \\
\subfigure[]{\includegraphics[width=0.23\textwidth]{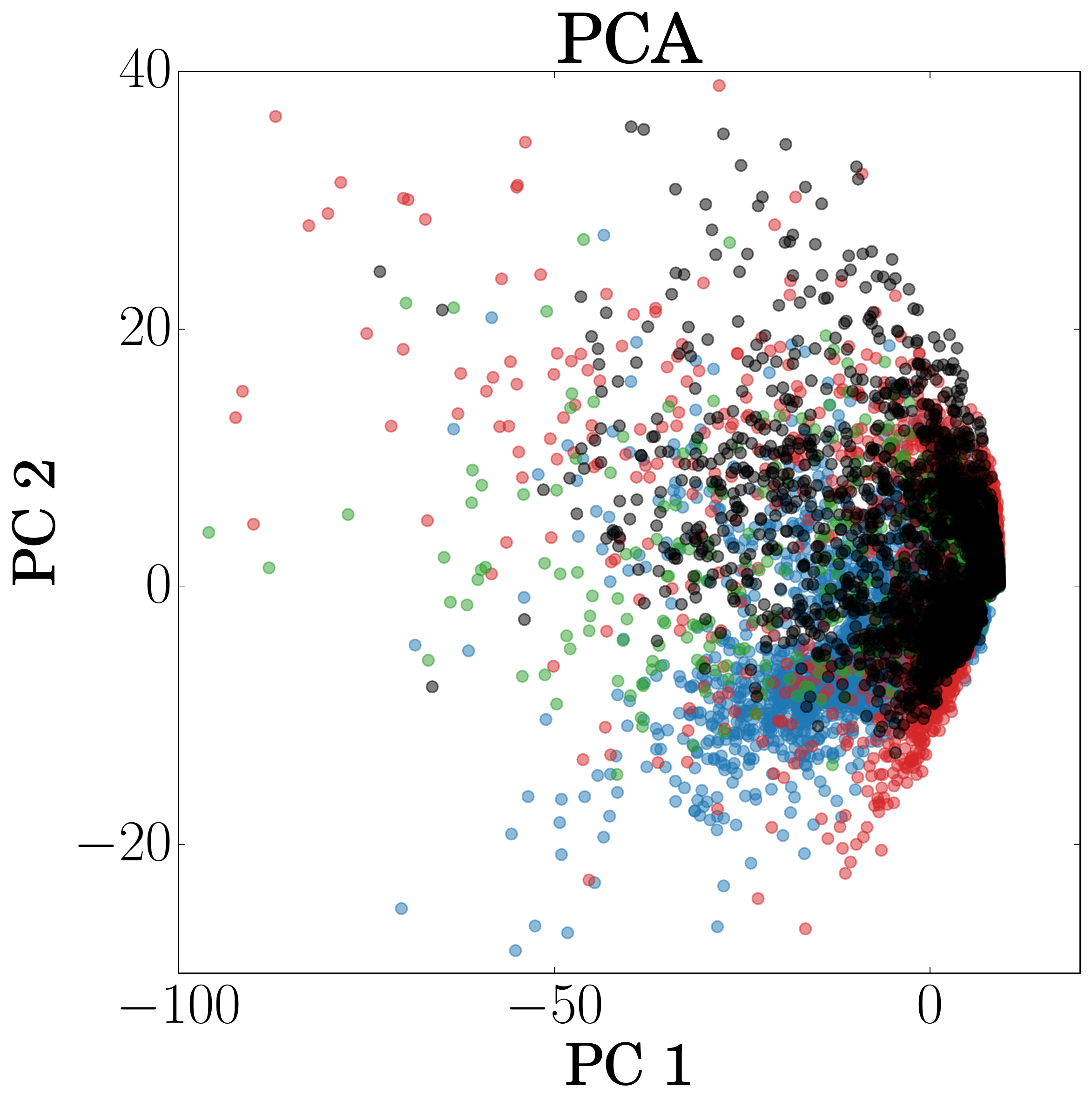}}
\subfigure[]{\includegraphics[width=0.224\textwidth]{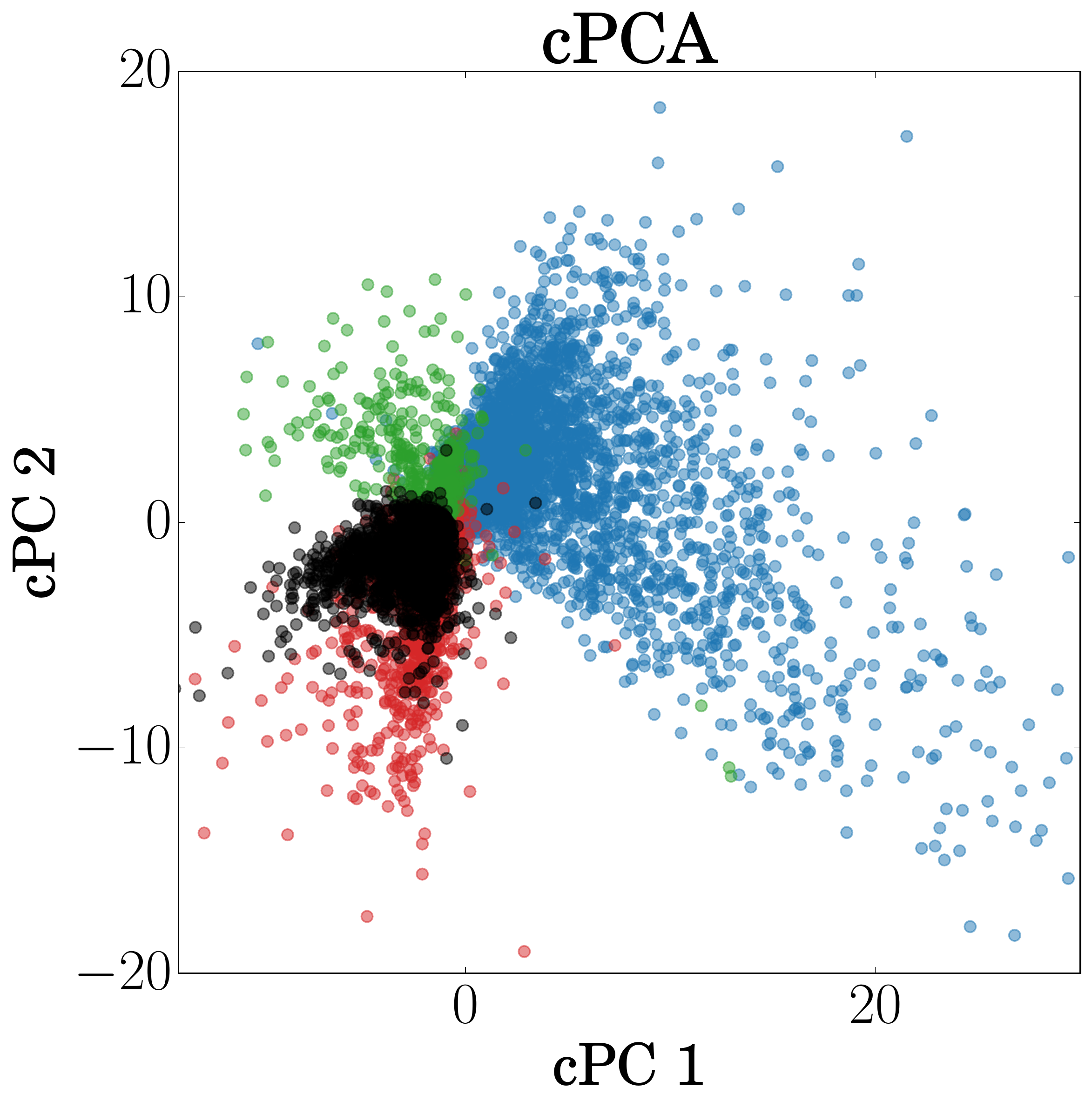}} 

\caption{\textbf{Discovering Subgroups in scRNA-Seq Data.} Here, we use PCA and cPCA to visualize a high-dimensional single cell RNA-Seq dataset in two dimensions. The dataset consists of four cell samples from two leukemia patients: a pre-transplant sample from patient 1 (\textbf{\textcolor{blue}{blue}}), a post-transplant from patient 1 (\textcolor{red}{\textbf{red}}), a pre-transplant sample from patient 2 (\textcolor{ForestGreen}{\textbf{green}}), and a post-transplant sample from patient 2 (\textcolor{black}{\textbf{black}}). (a) and (b) show the results using only the samples from patient 1, which demonstrate that cPCA more effectively separates the samples than PCA. When the cells from the second patient are included, in (c) and (d), again cPCA is more effective than PCA at separating the samples, although the post-transplant cells from both patients are similarly-distributed. We show plots of each sample separately in Appendix \ref{supp:full_results}, where it is easier to see the overlap between different samples.}
\label{fig:single_cell}
\end{figure}

% \abu{Make clear what the overlap is}

\paragraph*{Single Cell RNA-Seq of Leukemia Patients} We then analyze a higher-dimensional public dataset consisting of single-cell RNA expression levels of a mixture of bone marrow mononuclear cells (BMMCs) taken from a leukemia patient before stem cell transplant and BMMCs from the same patient after stem cell transplant (all single cell RNA-Seq data is obtained from \citet{Zheng2017}, and preprocessed using similar methods as described by the authors. In particular, before applying PCA or cPCA, all datasets are reduced to 500 genes, which are selected on the basis of highest dispersion [variance divided by mean] within the target data). Again, we perform PCA to see if we can visually discover the two samples in the transformed data. As shown in Fig. \ref{fig:single_cell}a, both cell types follow a similar distribution in the space spanned by the first two PCs. This is likely because the differences between the samples is small and the PCs instead reflect the heterogeneity of various kinds of cells within each sample or even variations in experimental conditions, which can have a significant effect on single-cell RNA-seq measurements \citep{Bhargava2014}. 

So we apply cPCA using a background dataset that consists of RNA-Seq measurements from a healthy individual's BMMC cells. We expect that this background dataset to contain the variation due to heterogeneous population of cells as well as variations in experimental conditions. We may hope, then, that contrastive PCA might be able to recover directions that are enriched in the target data, corresponding to pre- and post-transplant differences. Indeed, that is what we find, as shown in Fig. \ref{fig:single_cell}b. 

Next, we augment our target dataset with BMMC samples from a second leukemia patient, again before and after stem cell transplant. Thus, there are a total of four subpopulations of cells. Application of PCA on this data shows that the four subpopulations are not separable in the subspace spanned by the top 2 principal components (PCs), as shown in Fig. \ref{fig:single_cell}c. Again, however, when contrastive PCA is applied with the same background dataset, at least three of the subpopulations show much stronger separation (Fig. \ref{fig:single_cell}d). The cPCA embedding also suggests that the cell samples from both patients are more similar to each other after stem-cell transplant (red and black dots) than before the transplant (blue and green dots), a reasonable hypothesis which can be tested by the investigator. We see that cPCA can be a useful tool to infer the relationship between subpopulations, a topic we return to in Section \ref{subsection:ancestry}.

\subsection{Visualizing Important Relationships Within Data}
\label{subsection:ancestry}
\setcounter{footnote}{1}
In previous examples, we have seen that contrastive PCA allows the user to discover subclasses within a target dataset that are not labeled \textit{a priori}. However, even when subclasses are known ahead of time, dimensionality reduction can be a useful way to visualize the relationship within groups. For example, PCA is often used to visualize the relationship between ethnic populations based on genetic variants, because projecting the genetic variants into two dimensions often produces ``maps" that offer striking visualizations of geographic and historic trends  \citep{cavalli1998dna, Novembre2008}. But again, standard PCA is limited to identifying the most dominant structure; when this represents universal or uninteresting variation, cPCA can be more effective at visualizing trends.

\paragraph*{Mexican Ancestry.} The dataset that we use for this example consists of single nucleotide polymorphisms (SNPs) from the genomes of individuals from 5 states in Mexico, collected by the authors \citet{SilvaZolezzi2009}. Mexican ancestry is challenging to analyze using PCA since the PCs usually do not reflect geographic origin within Mexico; instead, they reflect the \textit{proportion} of European/Native American heritage of each Mexican individual, which dominates and obscures differences due to geographic origin within Mexico (see Fig. \ref{fig:ancestry}a). To overcome this problem, population geneticists prune SNPs, removing those known to derive from Europeans ancestry, before applying PCA. However, this procedure is of limited applicability since it requires knowing the origin of the SNPs and that the source of background variation to be very different from the variation of interest, which are often not the case.

As an alternative, we use contrastive PCA with a background dataset that consists of individuals from Mexico and from Europe\footnote{To avoid storing large covariance matrices in memory, dimensionality reduction using PCA was performed before the cPCA step. All explained variance was preserved because $n \ll d$.}. This background is dominated by Native American/European variation\footnote{It may be even more effective to use a background of Native Americans and Europeans, but this data was not available.}, allowing us to isolate the intra-Mexican variation in the target dataset. The results of applying cPCA are shown in Fig. \ref{fig:ancestry}b. We find that individuals from the same state in Mexico are embedded closer together. Furthermore, the two groups that are the most divergent are the Sonorans and the Mayans from Yucatan, which are also the most geographically distant within Mexico, while Mexicans from the other three states are close to each other, both geographically as well as in the embedding captured by contrastive PCA (see Fig. \ref{fig:ancestry}c).  

\begin{figure}[!tb]
\begin{center}
\subfigure[]{\includegraphics[width=0.23\textwidth]{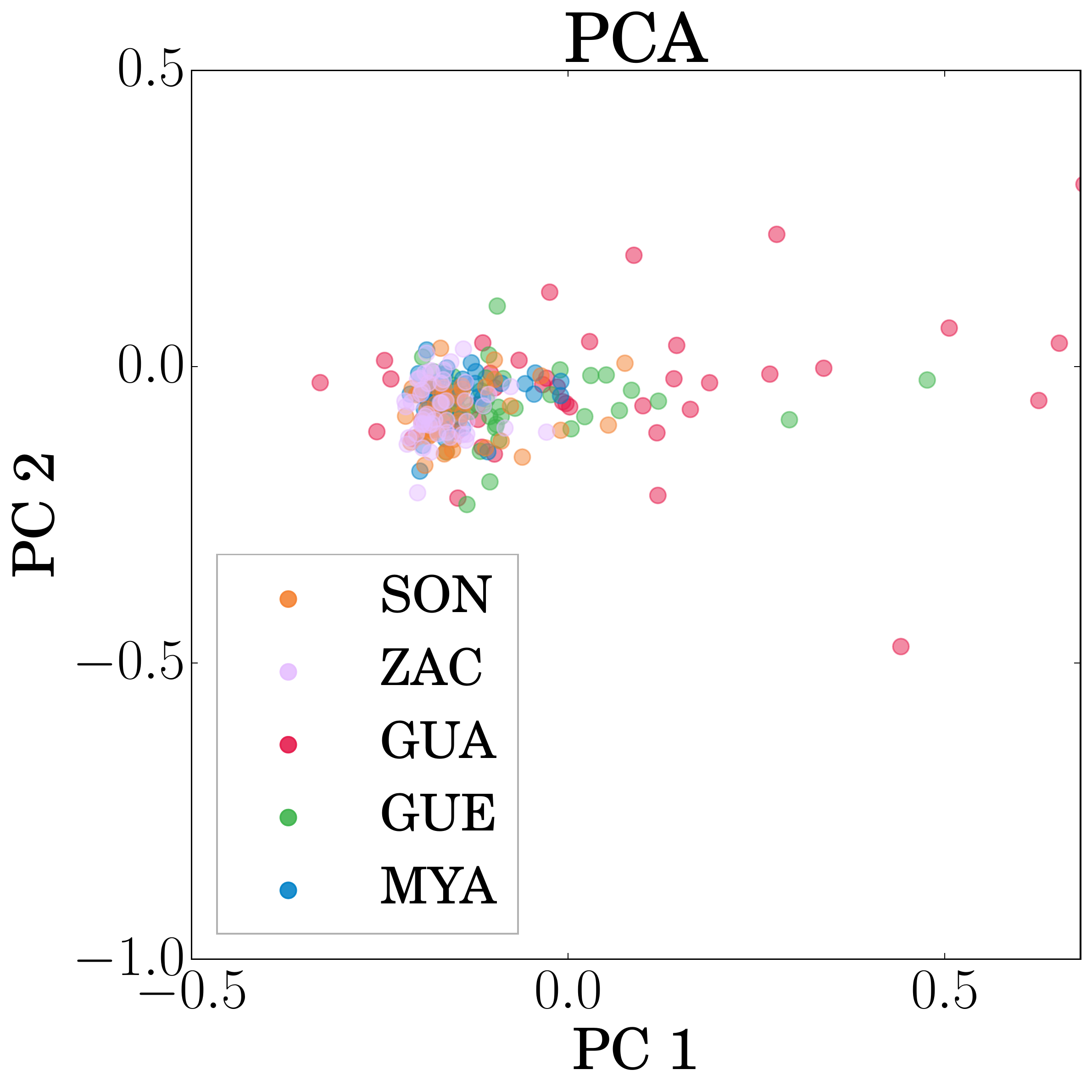}} 
\subfigure[]{\includegraphics[width=0.23\textwidth]{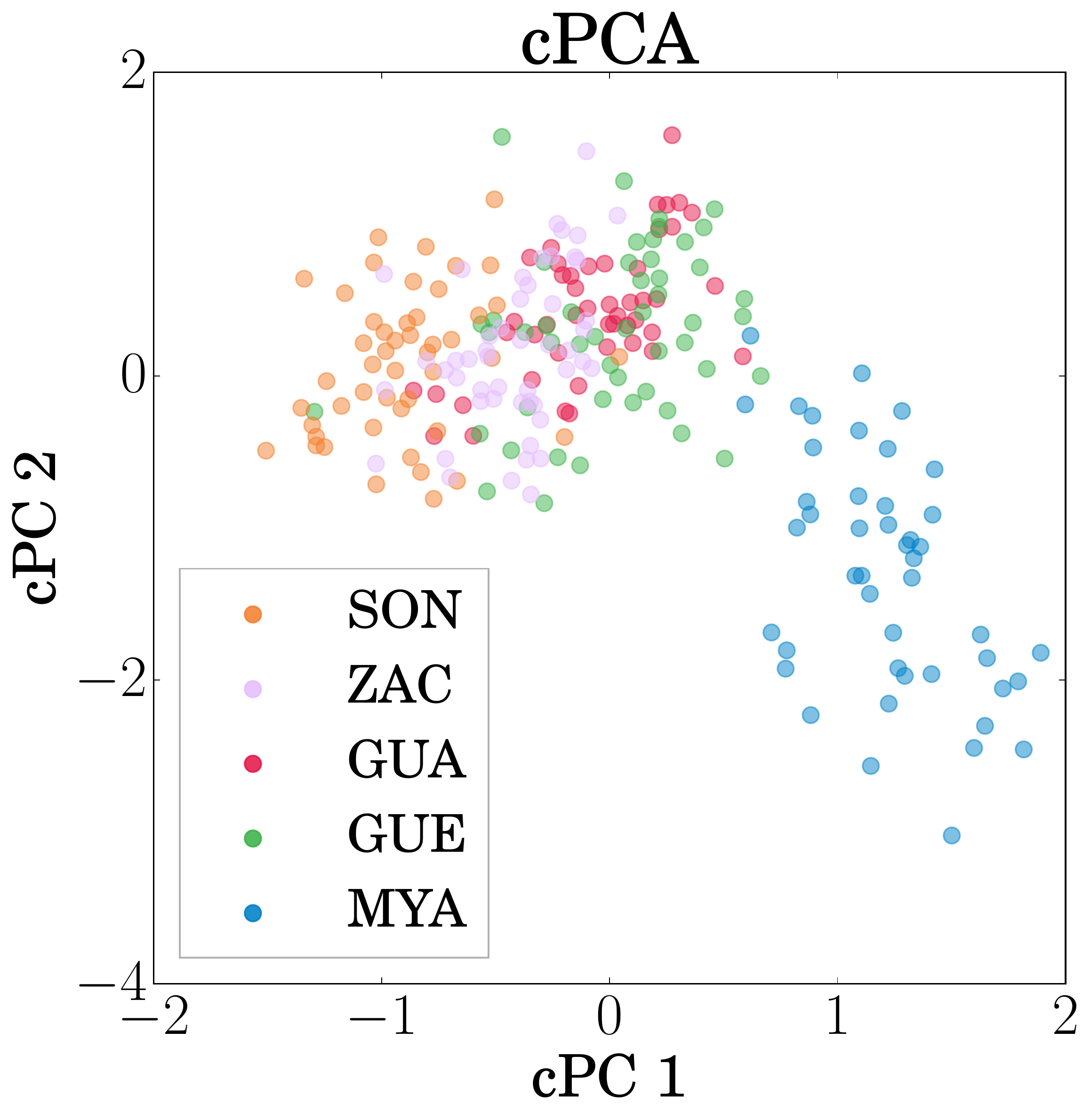}} \\
\subfigure[]{\includegraphics[width=0.35\textwidth]{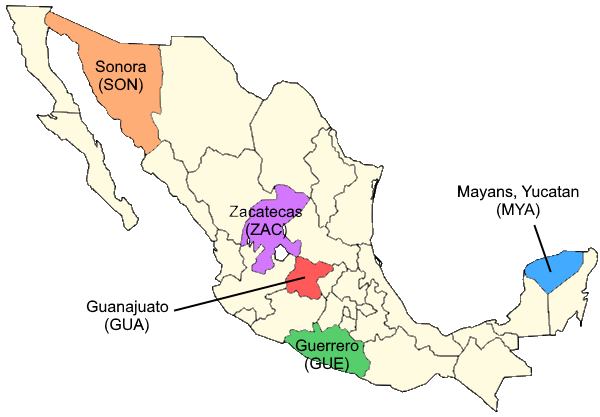}} 
\end{center}
\caption{\textbf{Relationship Between Mexican Ancestry Groups.} (a) Here, we find that PCA applied to genetic data from individuals from 5 Mexican states does not reveal any visually discernible patterns in the embedded data. (b) Contrastive PCA applied to the same dataset reveals patterns in the data: individuals from the same state are clustered closer together in the cPCA embedding. (c) Furthermore, the distribution of the points reveals relationships between the groups that matches the geographic location of the different states: for example, individuals from geographically-adjacent states are adjacent in the embedding. 
\label{fig:ancestry}
}

\end{figure}

\subsection{Feature Selection and Denoising}

We have seen that the cPCs are a useful bases for embedding data to discover subgroups. In addition, the cPCs can be examined directly, often yielding insight into what combination of features are the most relevant source of variation within the target.

We return to our first example, with the target dataset consisting of images of handwritten digits superimposed on grassy backgrounds (refer to Fig. \ref{fig:mnist_on_grass} for details). The contribution of each pixel to the first PC or cPC is indicated by the absolute value of the weight of each pixel in the component. We normalize the weights by squaring each weight and rescaling the squared weights to have a maximum value of 1. Fig. \ref{fig:denoise_mnist} (top) is a plot of the resulting weights for both PCA and cPCA. We find that PCA tends to emphasize pixels in the far left of the image and de-emphasize pixels in the center of the image, indicating that the source of most of the variance in the first PC is not due to the superimposed digits, but due to certain features in the background. On the contrary, cPCA tends to emphasize the pixels in the center, which allows it to distinguish images containing the digit 0 from images containing 1 quite easily.

Furthermore, we can use dimensionality reduction to denoise image in the target dataset by discarding all but the first few components, under the basic assumption that the non-leading components mostly consist of noise \citep{wold1987principal}. The result of both PCA and cPCA denoising of a representative image is shown Fig. \ref{fig:denoise_mnist} (bottom).

\begin{figure}[!tb]
\begin{center}
\subfigure{\includegraphics[width=0.40\textwidth]{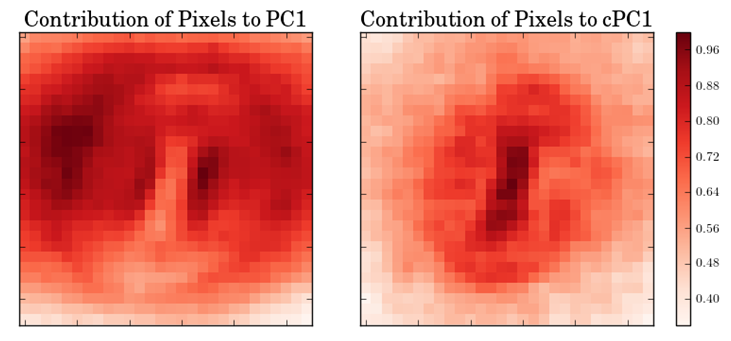}} \\
\subfigure{\includegraphics[width=0.45\textwidth]{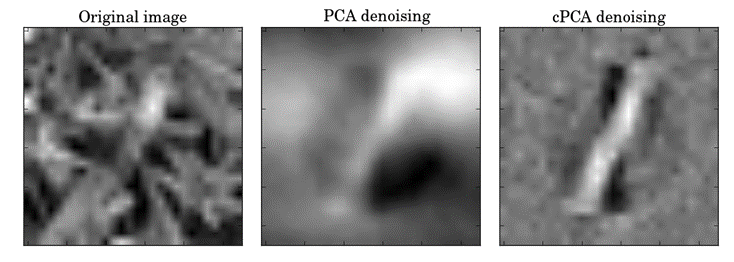}} 
\end{center}
\caption{\textbf{Features Captured by cPCA.} (Top) We look at the relative contribution of each pixel to the first PC and first cPC in the synthetic dataset described in Fig. \ref{fig:mnist_on_grass}. PCA tends to emphasize pixels in the far left of the image and de-emphasize pixels in the center of the image, indicating that most of the variance is due to background features. cPCA tends to emphasize the pixels in the center, suggesting that it identifies features corresponding to the superimposed digits. (Bottom) Thus, if we consider background features to be noise, we can use cPCA to denoise each image by discarding all but the first few (in this example, 10) cPCs. We do this on a representative image that consists of the digit 1 superimposed on a grassy background, recovering the digit clearly and discarding the grass in the background almost entirely. 
}
\label{fig:denoise_mnist}
\end{figure}

\subsection{Standardizing in a Data-Dependent Manner}

Because the PCA calculates the directions in a dataset with the highest covariance, it is highly sensitive to the units used to measure each feature. As a consequence, when different units are used to measure different features, it is common to standardize the data by dividing each column of the data matrix by its standard deviation, thereby ensuring that each feature has unit variance \citep{joliffe1992principal, wold1987principal}. However, this procedure has a drawback: noisy features with low variance are inflated to have the same variance as the most significant features; in fact, some sources suggest that standardization should not be used unless low-variance features are removed first \citep{vandenBerg2006}.

As an alternative, contrastive PCA can be used as a dimensionality-reduction technique directly, without standardization, in cases when a reference, signal-free dataset is available as a background. By searching for features that contrast between the target and background, cPCA automatically provides a data-dependent standardization by eliminating those features that are equally noisy in both the target and background. We illustrate this with an example.

\paragraph*{MHealth Measurements.} The MHealth public dataset \citep{Banos2015} consists of measurements from a variety of sensors (e.g. accelerometers, EKG, and gyroscopes) when subjects perform a series of different activities. In this example, our target dataset consists of sensor readings from a subject who is, at times, jogging and, at times, performing squats -- two very different activities. We may wonder whether the sensor data can be used to visually distinguish these two activities. In Fig. \ref{fig:mhealth}a, we show the result of applying PCA on the \textit{unstandardized} data: the two activities cannot be distinguished visually.

We then take as a background dataset sensor readings from  the subject when the subject is lying still. We assume this to be a signal-free reference, because most sensor readings will reflect their baseline  noise levels. By performing cPCA, we see the two activities resolve clearly into two separate subgroups, as shown in Fig. \ref{fig:mhealth}b -- with no standardization needed. For this experiment, a larger range of initial values of $\alpha$ was used as a parameter for Algorithm \ref{alg:cpca2} (see Appendix \ref{supp:full_results} for details).

\begin{figure}[!tb]
\subfigure[]{\includegraphics[width=0.23\textwidth]{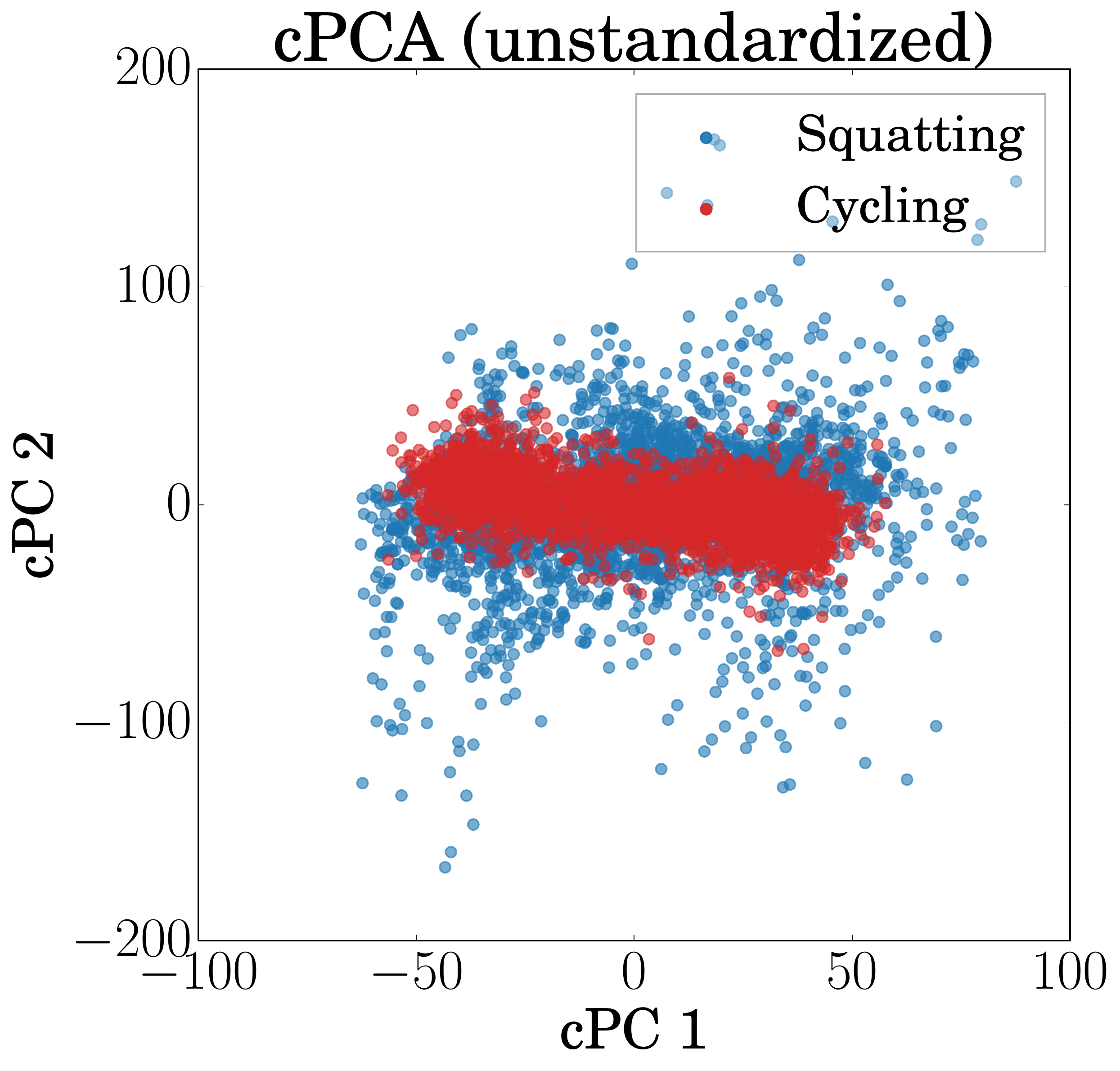}}
\subfigure[]{\includegraphics[width=0.224\textwidth]{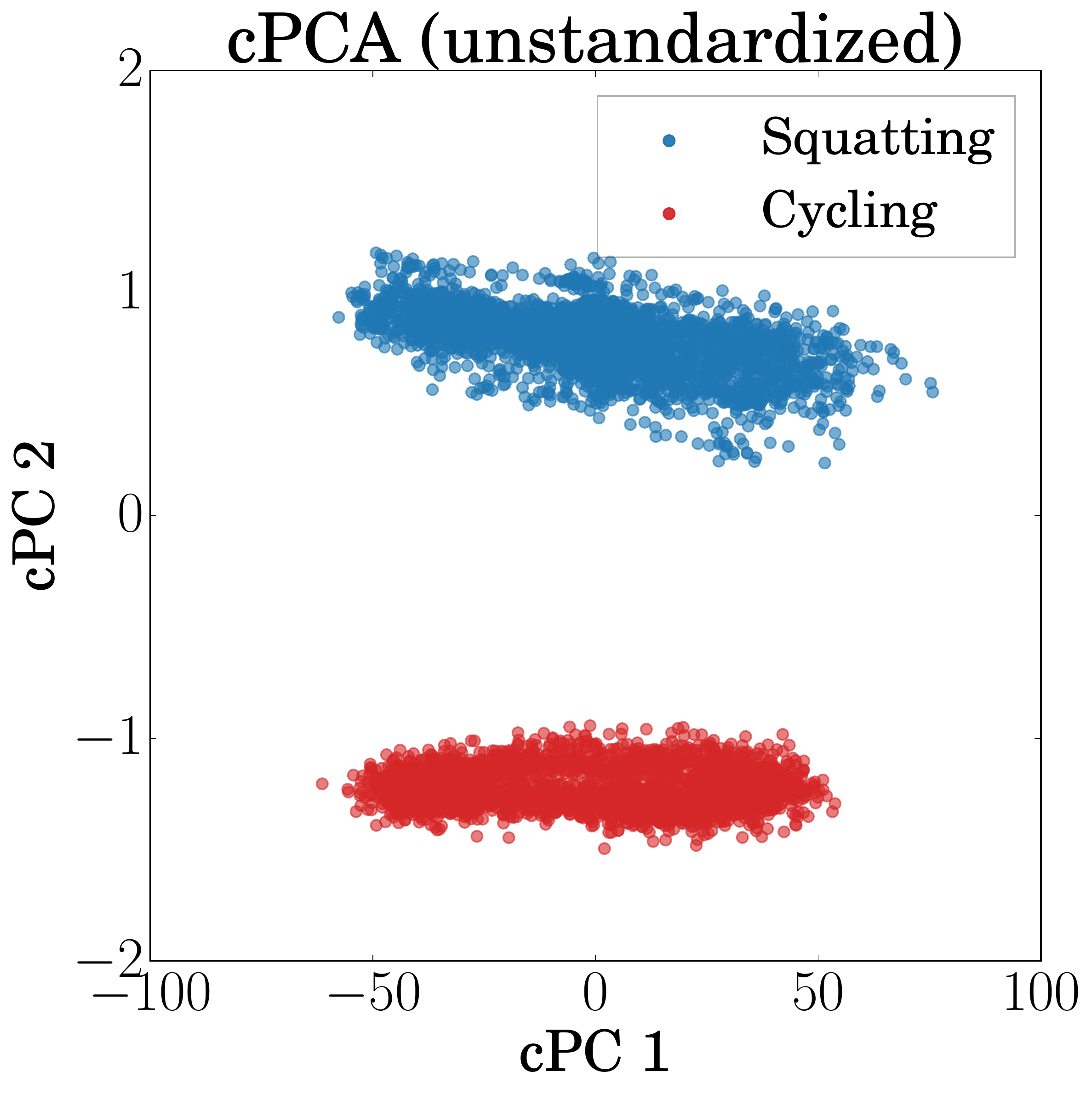}}

\caption{\textbf{Data-Dependent Standardization.} (a) Here, we see that PCA applied to the unstandardized MHealth dataset produces components that do not separate two very distinct activities, squatting and cycling, from each other. Further analysis, not shown here, reveals that is because the PCs are dominated by a few very noisy features. (b). However, contrastive PCA applied also to the unstandardized dataset is able to find a component along which the two activities are quite distinct.}
\label{fig:mhealth}
\end{figure}

\section{Theoretical Analysis}
\label{section:theory}

\begin{figure}
\centering
\subfigure[]{\includegraphics[width=0.23\textwidth]{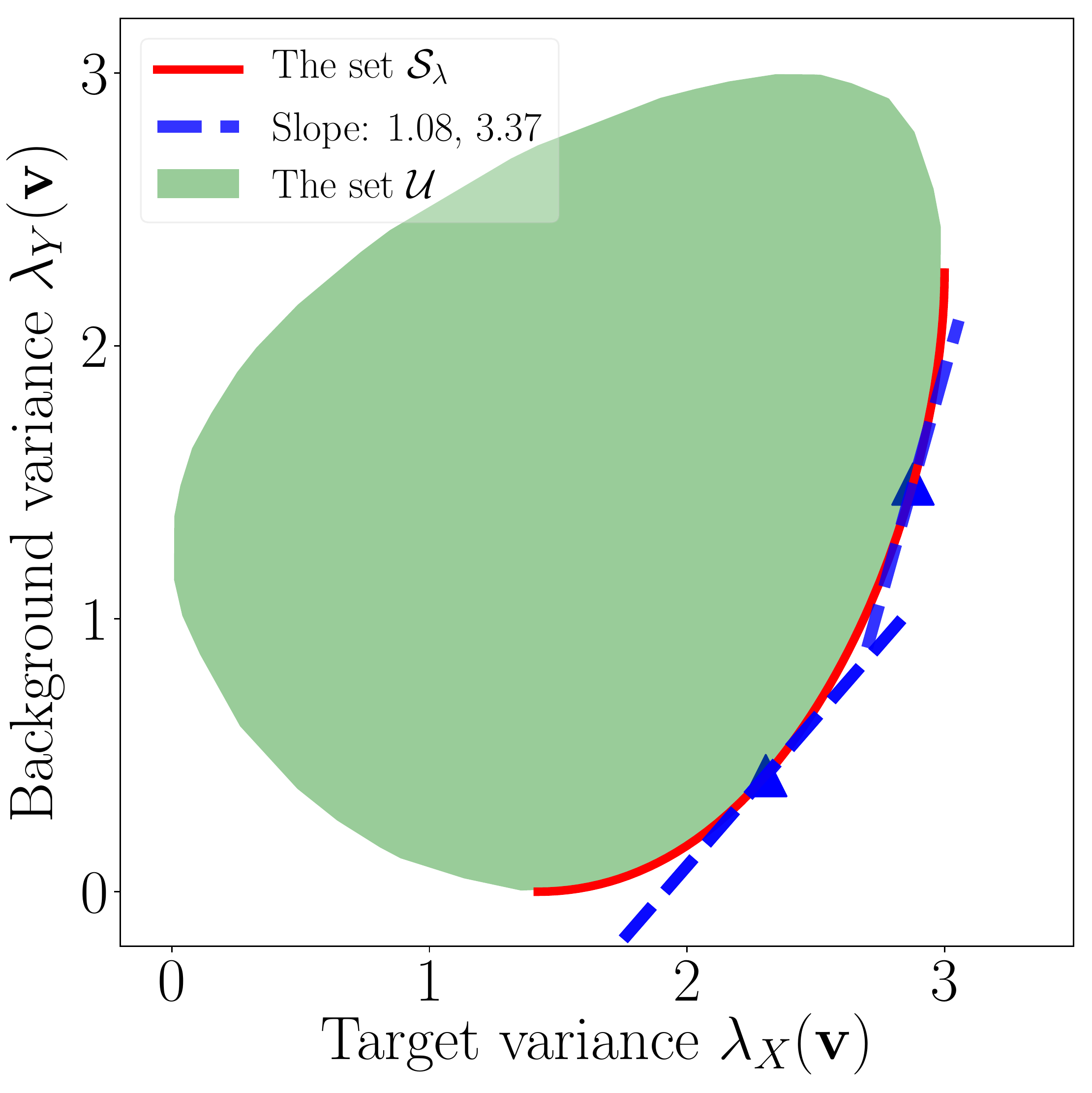}}
\subfigure[]{\includegraphics[width=0.23\textwidth]{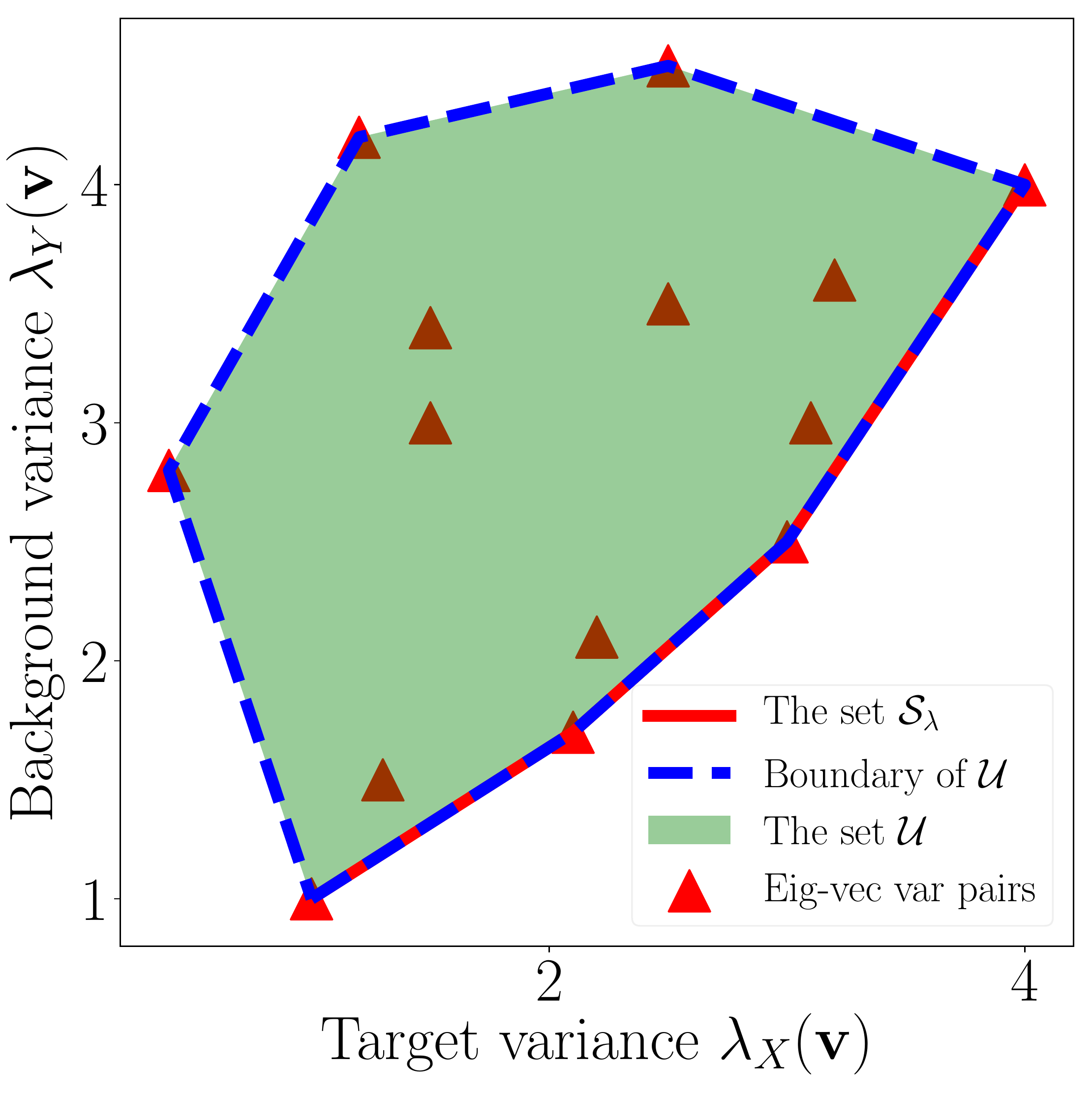}}
\caption{(a) Target-background variance pairs for randomly generated $C_X$ and $C_Y$. The green dots are samples from $\mathcal{U}$, whose lower-right boundary is the red curve $\mathcal{S}_\lambda$. 
The blue triangles are variance pairs for directions selected by \texttt{cPCA} with $\alpha=0.92,0.29$.
They are the points of tangency of the red curve and the tangent lines with slope $\frac{1}{\alpha}=0.08,3.37$ respectively.
(b) Target-background variance pairs for randomly generated simultaneously diagonalizable $C_X$ and $C_Y$. 
The green dots, samples from $\mathcal{U}$, consist the convex hull of the red triangles, which are variance pairs for the common eigenvectors of $C_X$ and $C_Y$. 
$\mathcal{S}_\lambda$ is the lower-right red line, and the boundary of $\mathcal{U}$ is the blue dashed line.}
\label{fig:theo}
\end{figure}

For any direction $\mathbf{v}\in \mathbb{R}^d_{\text{unit}}$, its target-background variance pair $(\lambda_X(\mathbf{v}), \lambda_Y(\mathbf{v}))$ fully determines its significance for contrastive PCA. Intuitively, we might say that for any two directions $\mathbf{v}_1, \mathbf{v}_2 \in \mathbb{R}^d_{\text{unit}}$, $\mathbf{v}_1$ is a better contrastive direction than $\mathbf{v}_2$ if it has a larger target variance and a smaller background variance. Let us formalize this notion: 
\begin{definition}
(Contrastiveness\label{def:contrastiveness})
For two directions $\mathbf{v}_{1}, \mathbf{v}_{2} \in \mathbb{R}^d_{\text{unit}}$, 
$\mathbf{v}_{1}$ is more contrastive than $\mathbf{v}_{2}$ w.r.t. the target and the background covariance matrices $C_X$ and $C_Y$, written as $\mathbf{v}_{1} \succ \mathbf{v}_{2} $, if one of the following is true:
\begin{align*}
& (1)~~~~\lambda_X(\mathbf{v}_1)\geq \lambda_X(\mathbf{v}_2),~\text{\textup{and}}~\lambda_Y(\mathbf{v}_1)< \lambda_Y(\mathbf{v}_2) \\
& (2)~~~~\lambda_X(\mathbf{v}_1)> \lambda_X(\mathbf{v}_2),~\textup{and}~\lambda_Y(\mathbf{v}_1)\leq \lambda_Y(\mathbf{v}_2).
\end{align*}
\end{definition}
We should note that the above definition provides a partial order of the directions in $\mathbb{R}_{unit}^d$. 
Then it is natural to say a direction $\mathbf{v}$ is most contrastive if there are no other directions more contrastive than $\mathbf{v}$.
Formally,
\begin{definition} 
\label{def:most_contrastive}
Define the set of most contrastive directions $\mathcal{S}_\mathbf{v}$ and the corresponding set of target-background variance pairs $\mathcal{S}_\lambda$ to be:
\begin{align*}
& \mathcal{S}_\mathbf{v} \eqdef \{\mathbf{v}\in \mathbb{R}^k_{\text{unit}} :  ~\nexists~\mathbf{v}'\in \mathbb{R}^d_{\text{unit}} ,~s.t.~\mathbf{v}' \succ \mathbf{v}\},\\
& \mathcal{S}_\lambda \eqdef \{ (\lambda_X(\mathbf{v}), \lambda_Y(\mathbf{v})): \mathbf{v} \in \mathcal{S}_\mathbf{v}\}.
\end{align*}
\end{definition}

It is also convenient to define $\mathcal{U}$ to be the set of target-background variance pairs for all directions in $\mathbb{R}_{\text{unit}}^d$, i.e. $\mathcal{U} \eqdef \{ (\lambda_X(\mathbf{v}), \lambda_Y(\mathbf{v})): \mathbf{v} \in \mathbb{R}^d_{\text{unit}}\}$. 
In order to illustrate the quantities defined above, we provide a toy example in Fig. \ref{fig:theo}a by randomly generating the matrices $C_X$ and $C_Y$.
In Fig. \ref{fig:theo}a, the green dots are samples of $\mathcal{U}$, and the red curve corresponds to elements in $\mathcal{S}_\lambda$. The reader will notice that $\mathcal{S}_\lambda$ forms the lower-right boundary of $\mathcal{U}$, which can also be inferred from the above definition. 

Now let us consider directions that are returned by \texttt{cPCA}. Without loss of generality, we will focus our attention on the top cPC selected by \texttt{cPCA} (for different values of $\alpha$).\footnote{This is because, after selecting the first $k$ contrastive components, the $(k+1)$-th contrastive component is obtained by maximizing $\mathbf{v}^T(C_X-\alpha C_Y)\mathbf{v}$ over the space orthogonal to the first $k$ components. By rotating the space such that the first $k$ components correspond to the first $k$ dimensions, and then truncating the first $k$ dimensions, the problem of selecting the $(k+1)$-th contrastive component is reduced to the same problem as finding the top contrastive component but with dimensionality $k-d$.}
For any contrastive analysis method to be reasonable, one would require that the directions it generates lie in $\mathcal{S}_\mathbf{v}$. We show that this is indeed the case for $\texttt{cPCA}$. Furthermore, we show that the set of top cPCs with different values of $\alpha$ is actually identical to $\mathcal{S}_\mathbf{v}$. In other words, \texttt{cPCA} recovers all contrastive directions, yielding its optimality. This is stated as below (with proof provided in Appendix \ref{supp:proof}):

\begin{theorem}\label{thrm:opt_cPCA}
Let $\mathcal{S}^{cPCA}_\mathbf{v}$ be the set of top contrastive components of \texttt{cPCA} and let $\mathcal{S}^{cPCA}_\lambda$ be the corresponding set of target-background variance pairs:
\begin{align*}
&\mathcal{S}^{cPCA}_\mathbf{v} \eqdef \{ \mathbf{v}: \exists \alpha>0~s.t.~\mathbf{v}=\argmax_{\mathbf{v}'\in\mathbb{R}^d_{\text{unit}}} \lambda(\mathbf{v}')-\alpha \sigma(\mathbf{v}')\}, \\
&\mathcal{S}^{cPCA}_\lambda\eqdef \{ (\sigma(\mathbf{v}), \lambda(\mathbf{v})): \mathbf{v} \in \mathcal{S}^{cPCA}_\mathbf{v}\}.
\end{align*}
For $\mathcal{S}_\mathbf{v}$, $\mathcal{S}_\lambda$ in Def. \ref{def:most_contrastive}, we have
\begin{align*}
\mathcal{S}^{cPCA}_\mathbf{v} = \mathcal{S}_\mathbf{v}, ~~~~\mathcal{S}^{cPCA}_\lambda=\mathcal{S}_\lambda.
\end{align*}
\end{theorem}

\begin{remark}(A geometrical interpretation of $\alpha$)

\normalfont
For the direction $\mathbf{v}$ selected by \texttt{cPCA} with the contrast parameter set to $\alpha$, its variance pair $(\lambda_X(\mathbf{v}),\lambda_Y(\mathbf{v}))$ corresponds to the point of tangency of $\mathcal{S}_\lambda$ with a line of slope $1/\alpha$. 
For example, the left blue triangle in Fig. \ref{fig:theo}a corresponds to the \texttt{cPCA} direction with $\alpha=0.92$, and it is the point of tangency of the red curve $\mathcal{V}(\mathcal{S})$ and the blue line with slope $1.08$. As a result, by varying $\alpha$ from zero to infinity, \texttt{cPCA} selects directions with variance pairs traveling from the lower-left to the upper-right end of $\mathcal{S}_\lambda$.

This interpretation can be derived from the following observation. Consider any sequence $\alpha_n \downarrow \alpha$. 
Then there exists a sequence $\mathbf{v}_n$ such that $\mathbf{v}_n$ is the solution to \eqref{eq:opt_obj} with alpha value $\alpha_n$, and $\lambda_X(\mathbf{v}_n)\uparrow \lambda_X(\mathbf{v})$, $\lambda_Y(\mathbf{v}_n)\uparrow
\lambda_Y(\mathbf{v})$.
By Lemma \ref{lm:cPCA_select_cond}, 
\begin{align*}
\frac{1}{\alpha_n} \leq \frac{\lambda_Y(\mathbf{v}_n)-\lambda_Y(\mathbf{v})}{\lambda_X(\mathbf{v}_n)-\lambda_X(\mathbf{v})} \leq \frac{1}{\alpha},
\end{align*}
giving 
\begin{align*}
\lim_{n\to\infty} \frac{\lambda_Y(\mathbf{v}_n)-\lambda_Y(\mathbf{v})}{\lambda_X(\mathbf{v}_n)-\lambda_X(\mathbf{v})}  = \frac{1}{\alpha}.
\end{align*}
This implies that $(\lambda_X(\mathbf{v}),\lambda_Y(\mathbf{v}))$ is the point of tangency of $\mathcal{S}_\lambda$ and the slope-$\frac{1}{\alpha}$ tangent line.
\end{remark}

\begin{example} (Simultaneously diagonalizable matrices)

\normalfont
A closed form representation of $\mathcal{S}_\lambda$ can be derived for the special case where the matrices $C_X$ and $C_Y$ are simultaneously diagonalizable. We derive it here to provide some intuition for the topology of target-background variance pairs.

Let $Q$ be the unitary matrix that diagonalize $C_X$ and $C_Y$, i.e.
\begin{align*}
C_X = Q \Lambda_X Q^T, ~~~~~~ C_Y=Q \Lambda_Y Q^T,
\end{align*}
where $\Lambda_X = \textrm{diag}(\lambda_{X,1},\cdots,\lambda_{X,d})$, $\Lambda_Y = \textrm{diag}(\lambda_{Y,1},\cdots,\lambda_{Y,d})$. 
Let $\mathbf{q}_1,\cdots,\mathbf{q}_d$ be the eigenvectors. 
Any unit vector can be written as $\mathbf{v} = \sum_{i} \sqrt{c_i} \mathbf{q}_i$, for $c_1,\cdots,c_d \geq 0$, $\sum_i c_i=1$. 
Then the target and the background variances can be written as 
\begin{align*}
& \lambda_X(\mathbf{v}) = \mathbf{v}^T C_X \mathbf{v} = \sum_i c_i \lambda_{X,i},\\
&\lambda_Y(\mathbf{v})= \mathbf{v}^T C_Y \mathbf{v} = \sum_i c_i \lambda_{Y,i}.
\end{align*}
Since the variance pair $(\lambda_X(\mathbf{v}), \lambda_Y(\mathbf{v}))$ is a convex combination of the variance pairs of eigenvectors $\{(\lambda_{X,i},\lambda_{Y,i})\}_{i=1}^d$, the set of variance pairs $\{(\lambda_X(\mathbf{v}), \lambda_Y(\mathbf{v})): \mathbf{v}\in\mathbb{R}^d_{\text{unit}}\}$ is the convex hull of $\{(\lambda_{X,i},\lambda_{Y,i})\}_{i=1}^d$. 
Also $\mathcal{S}_\lambda$ is the lower-right boundary of the convex hull of $\{(\lambda_{X,i},\lambda_{Y,i})\}_{i=1}^d$. 
We visualize this n Fig. \ref{fig:theo} (b) using randomly generated the simultaneously diagonalizable matrices $C_X$ and $C_Y$. 

As a result, $\mathcal{S}_\mathbf{v}$ can be written as follows.
Let $\mathbf{q}_{(1)},\cdots, \mathbf{q}_{(k)} \in \{\mathbf{q}_i\}_{i=1}^d$ be the eigenvectors whose variance pairs $(\lambda_{X,(j)},\lambda_{Y,(j)})$ lie on the lower-right boundary of the convex hull of $\{(\lambda_{X,i},\lambda_{Y,i})\}_{i=1}^d$, indexed in the ascending order of $\lambda_{X,(j)}$.
Then 
\begin{align*}
& \mathcal{S}_\mathbf{v} = \{ \mathbf{v}: \mathbf{v} = \sqrt{c} \mathbf{q}_{(j)} + \sqrt{1-c} \mathbf{q}_{(j+1)},\\&~for~0\leq c\leq 1, 1\leq j \leq k-1 \}.
\end{align*}
This implies that $\mathcal{S}_\mathbf{v}$ is a union of $(k-1)$ curved line segments of the form $\sqrt{c} \mathbf{q}_{(j)} + \sqrt{1-c} \mathbf{q}_{(j+1)}$, which is itself a curved line segment in the $k$ dimensional subspace spanned by $\mathbf{q}_{(1)},\cdots, \mathbf{q}_{(k)}$.
\end{example}
 
\section{Extensions: Kernel cPCA} 
We extend cPCA to Kernel cPCA, following the analogous extension of PCA to kernel PCA \cite{scholkopf1997kernel}. Full details are in Appendix \ref{suppsec: kernelcPCA}.

Consider the nonlinear transformation $\Phi: \mathbb{R}^d \mapsto F$ that maps the data to some feature space $F$. We assume that the mapped data, $\Phi(X_1),\cdots,\Phi(X_n)$, $\Phi(Y_1),\cdots,\Phi(Y_m)$, is centered, i.e. $\sum_{i=1}^n \Phi(X_i)=\sum_{j=1}^m \Phi(Y_j)=0$. (The general case is considered in Supp. \ref{suppsec: kernelcPCA}.)

The covariance matrices for the target and the background can be written as 
\begin{align*}
\bar{C}_X = \frac{1}{n} \sum_{i=1}^n \Phi(X_i)\Phi(X_i)^T,~~~~\bar{C}_Y = \frac{1}{m} \sum_{j=1}^m \Phi(Y_j)\Phi(Y_j)^T. 
\end{align*}
Contrastive PCA on the transformed data solves for the eigenvectors of $(\bar{C}_X-\alpha \bar{C}_Y) \mathbf{v}$, where the $k$-th eigenvector is the $k$-th contrastive component, but this is in efficient if the dimensionality of $F$ is large. 

We next describe the kernel cPCA algorithm, which allows us to efficiently perform contrastive analysis on the transformed data. 

Let $N=n+m$ and denote the data as $(Z_1,\cdots,Z_N)=(X_1,\cdots,X_n,Y_1,\cdots,Y_m)$.
Define the kernel matrix $K$ to have the $ij$-th element $K_{ij}=\Phi(Z_i) \cdot \Phi(Z_j)$, and write it in form of a block matrix as 
\begin{align}
K=\left[\begin{array}{cc}
K_{X}&K_{XY} \\ K_{YX} &K_{Y}
\end{array}\right],
\label{eq:kercPCA5}
\end{align}
where $K_{X}\in\mathbb{R}^{n\times n}$, $K_{Y}\in\mathbb{R}^{m\times m}$ are the sub-kernels corresponding to $X_1,\cdots,X_n$, and $Y_1,\cdots,Y_m$, respectively. 

As derived in Appendix \ref{suppsec: kernelcPCA}, instead of directly calculating the eigenvectors of $(\bar{C}_X-\alpha \bar{C}_Y) \mathbf{v}$, we can consider its dual representation $\mathbf{v}=\sum_{i=1}^N a_i \Phi(Z_i)$, and solve $a_i$'s via the following eigenvalue problem for non-zero eigenvalues:
\begin{align}\label{eq:kercPCA5}
\lambda \mathbf{a} = \tilde{K} \mathbf{a},
\end{align}
where the first eigenvector $\mathbf{a}^{(1)}$ corresponds to the first contrastive component, and
\begin{align*}
\tilde{K}=\left[\begin{array}{cc}
\frac{1}{n} K_{X}& \frac{1}{n} K_{XY} \\ -\frac{\alpha}{m} K_{YX}  & -\frac{\alpha}{m} K_{Y}
\end{array}\right].
\end{align*}
To make $\Vert \mathbf{v} \Vert=1$, we require $\mathbf{a}^T K \mathbf{a} =1$.
Finally, we can project the data onto the $k$-th contrastive component by
\begin{align*}
[\mathbf{v}^{(k)} \cdot \Phi(Z_1),\cdots,\mathbf{v}^{(k)} \cdot \Phi(Z_N)] = K \mathbf{a}^{(k)}.
\end{align*}
Note that in the above calculation, the kernel can be constructed via some kernel function $h(\cdot,\cdot)$ as $K_{ij}=h(Z_i,Z_j)$, and the projected data can be computed as $K \mathbf{a}^{(k)}$.
As a result, by Kernel cPCA, we can actually perform cPCA in the feature space without explicitly computing the non-linearly transformed data.

\begin{example} (Kernel cPCA: a toy example)

\normalfont
In this dataset, $d=10$, and the first two dimensions $X_1, X_2$ contain the subgroup structure in the target data.  As shown in Fig. \ref{fig:toy_kernelcPCA_data}a, the two subgroups can not be linearly separated directly. However, Fig. \ref{fig:toy_kernelcPCA_data}b shows that they can be linearly separated if we project the data on the non-linear features $\phi(X_1)=X_1^2$ and $\phi(X_2)=X_2^2$. 
% Also, the background data has a much smaller variance in the non-linear feature space than in the original space. 

\begin{figure}[]
\centering
\subfigure[]{\includegraphics[width=0.23\textwidth]{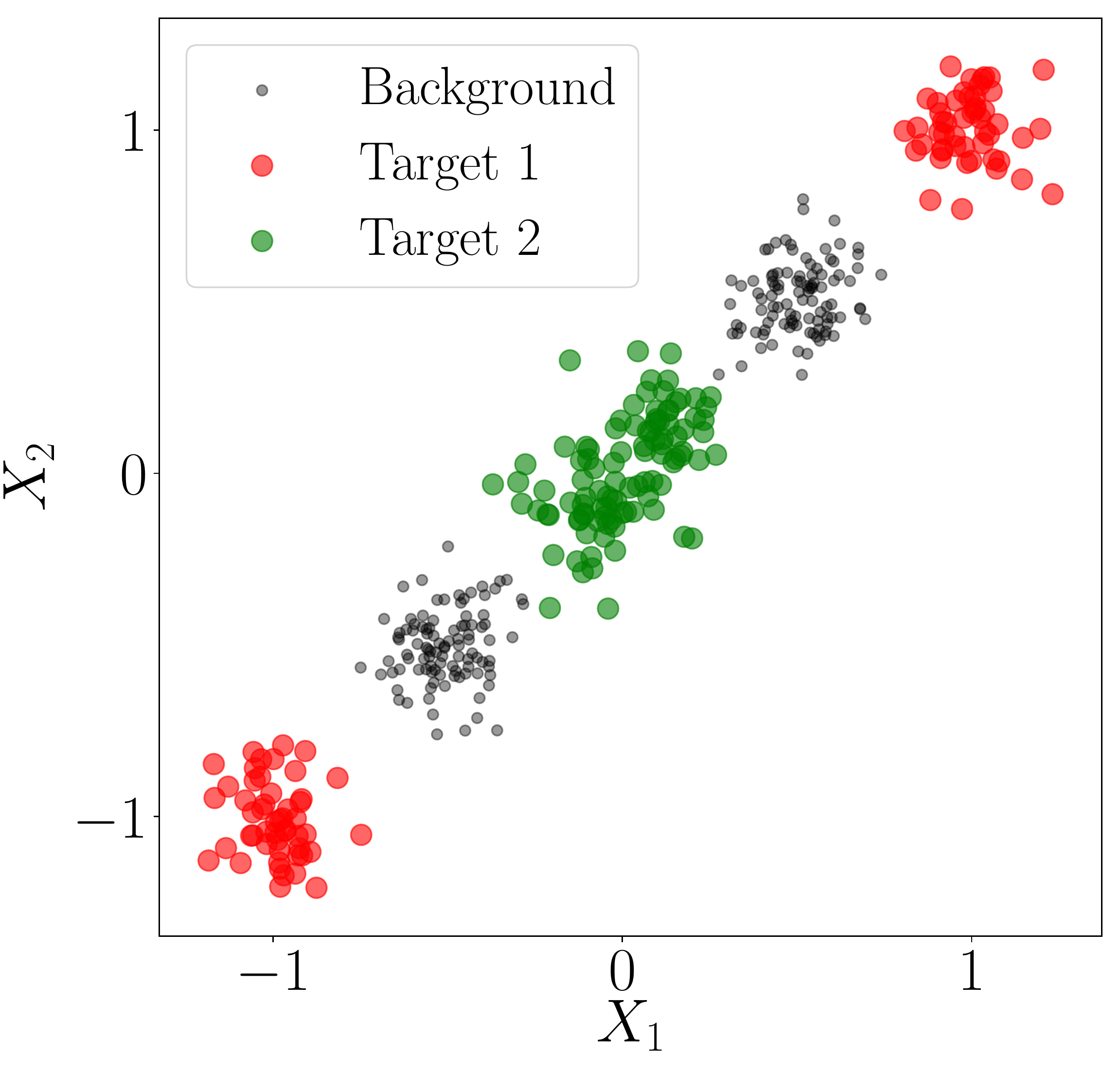}}
\subfigure[]{\includegraphics[width=0.23\textwidth]{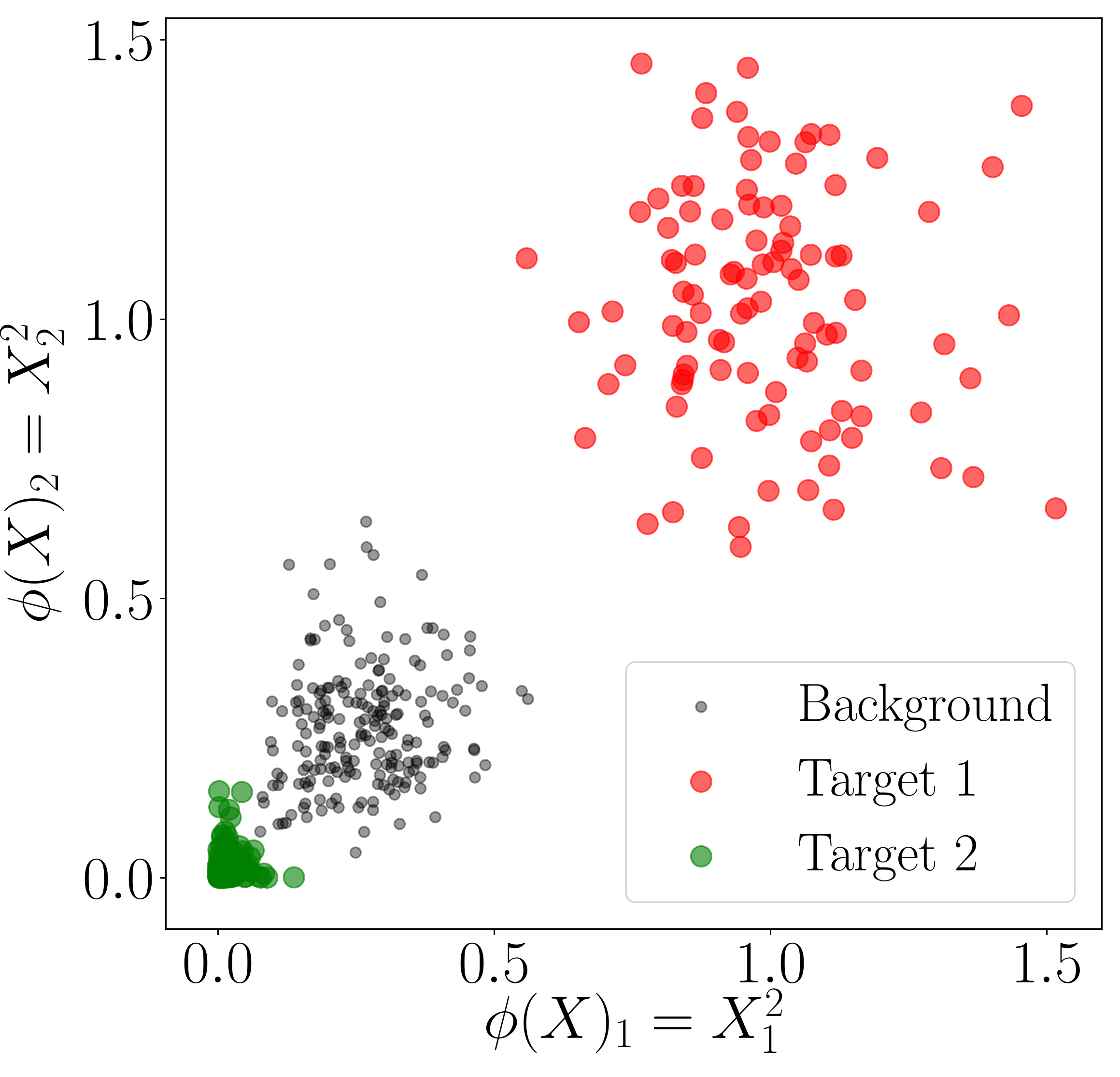}}
\caption{A toy example for kernel cPCA. (a) the data on the first two dimension $X_1,X_2$, and the two subgroups in the target data (red/ green) are not linearly separable. (b) the data on two non-linear features $\phi(X_1)=X_1^2, \phi(X_2)=X_2^2$, where the two subgroups become linearly separable. \label{fig:toy_kernelcPCA_data}}
\end{figure}

We tested PCA, cPCA, kernel PCA, kernel cPCA, using the polynomial kernel $K(\mathbf{X},\mathbf{Y})=(\mathbf{X}^T \mathbf{Y}+1)^2$ for the latter two to address the non-linear mapping. As shown in Fig. \ref{fig:toy_kernelcPCA_result}, both cPCA and kernel cPCA recover the subspace that contains the subgroup structure, but only Kernel cPCA produces a subspace where the two subgroups are linearly separable.
\end{example}

\begin{figure}
\centering
\subfigure[]{\includegraphics[width=0.23\textwidth]{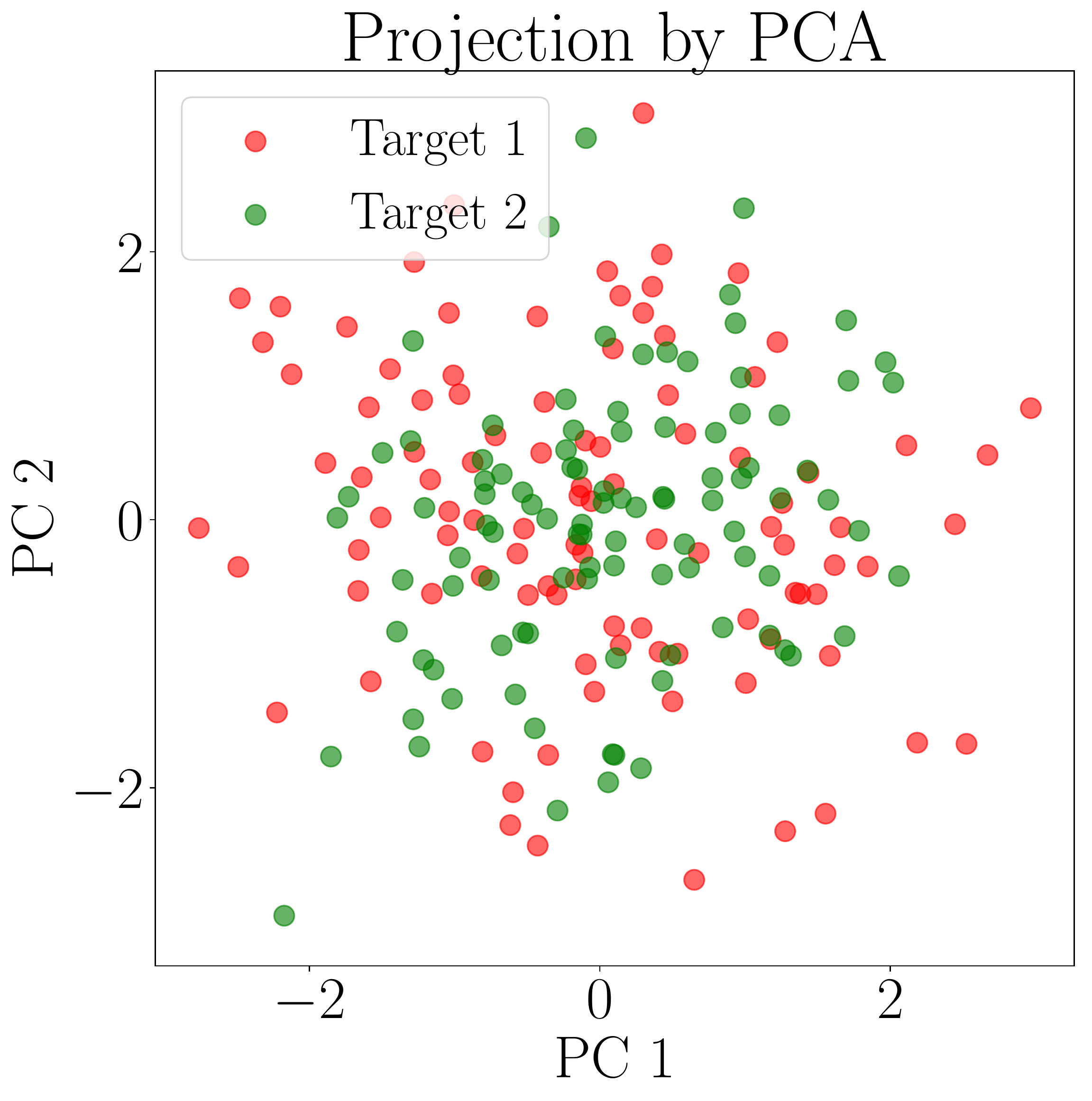}}
\subfigure[]{\includegraphics[width=0.23\textwidth]{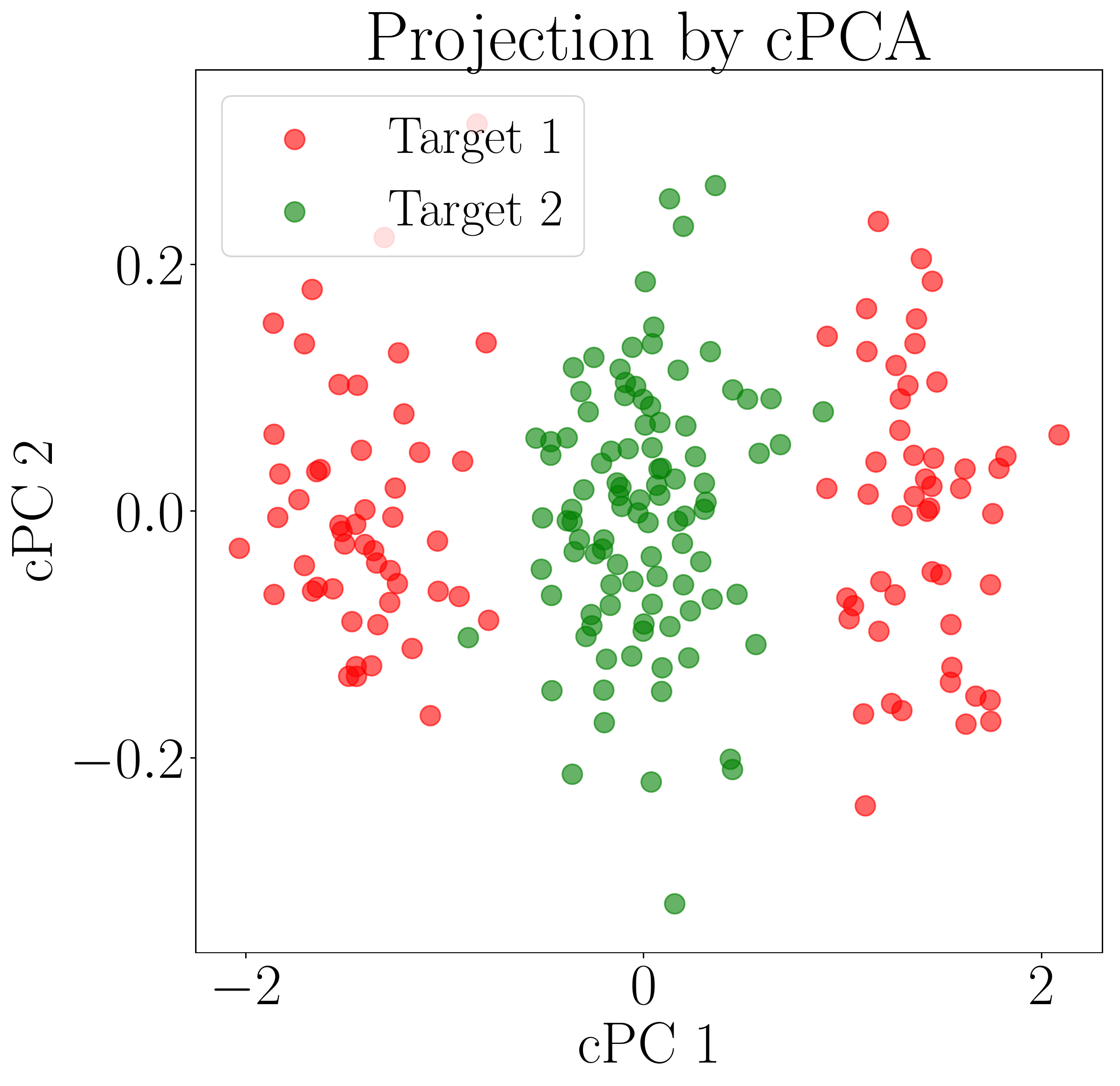}}\\
\subfigure[]{\includegraphics[width=0.23\textwidth]{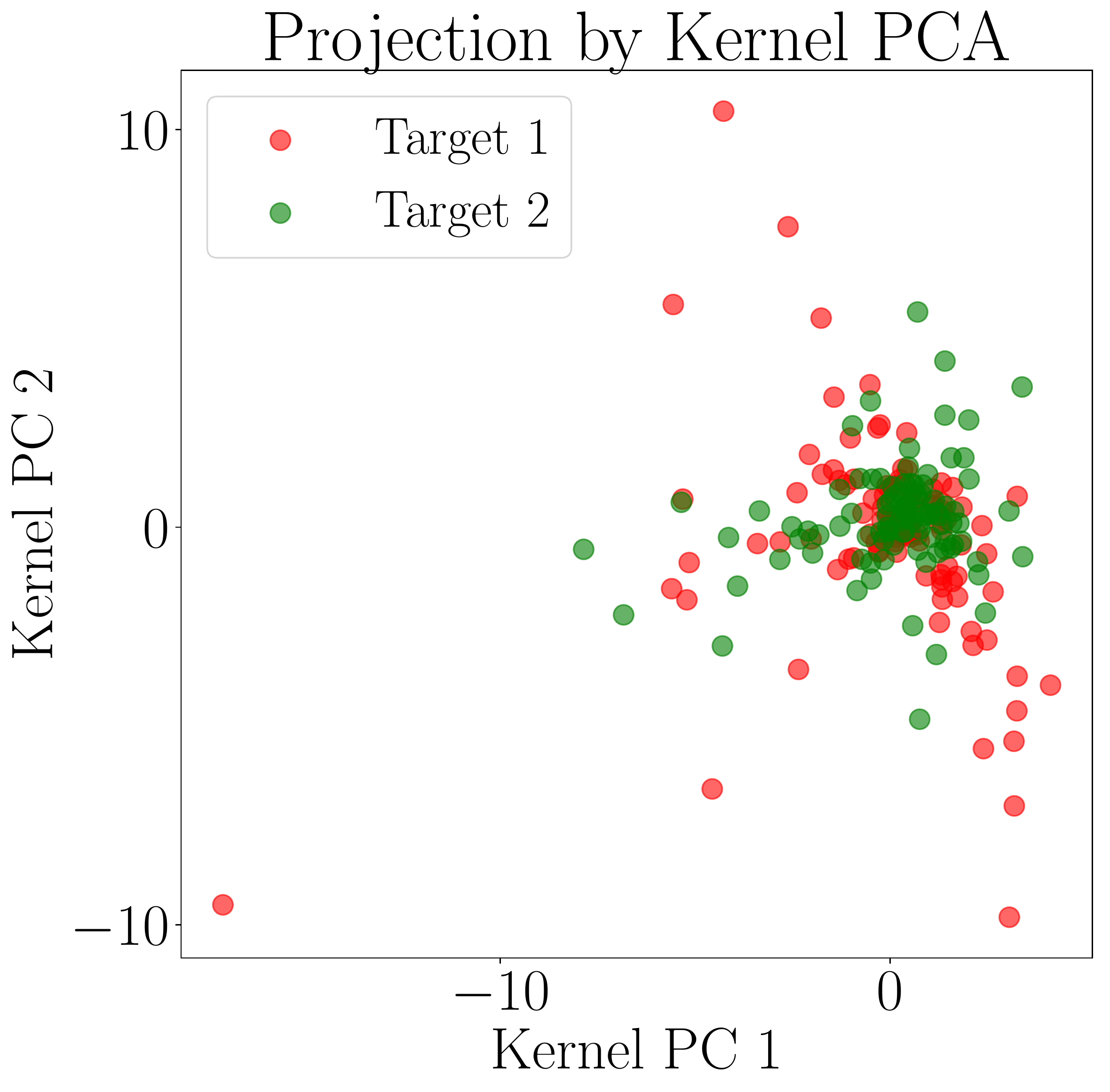}}
\subfigure[]{\includegraphics[width=0.23\textwidth]{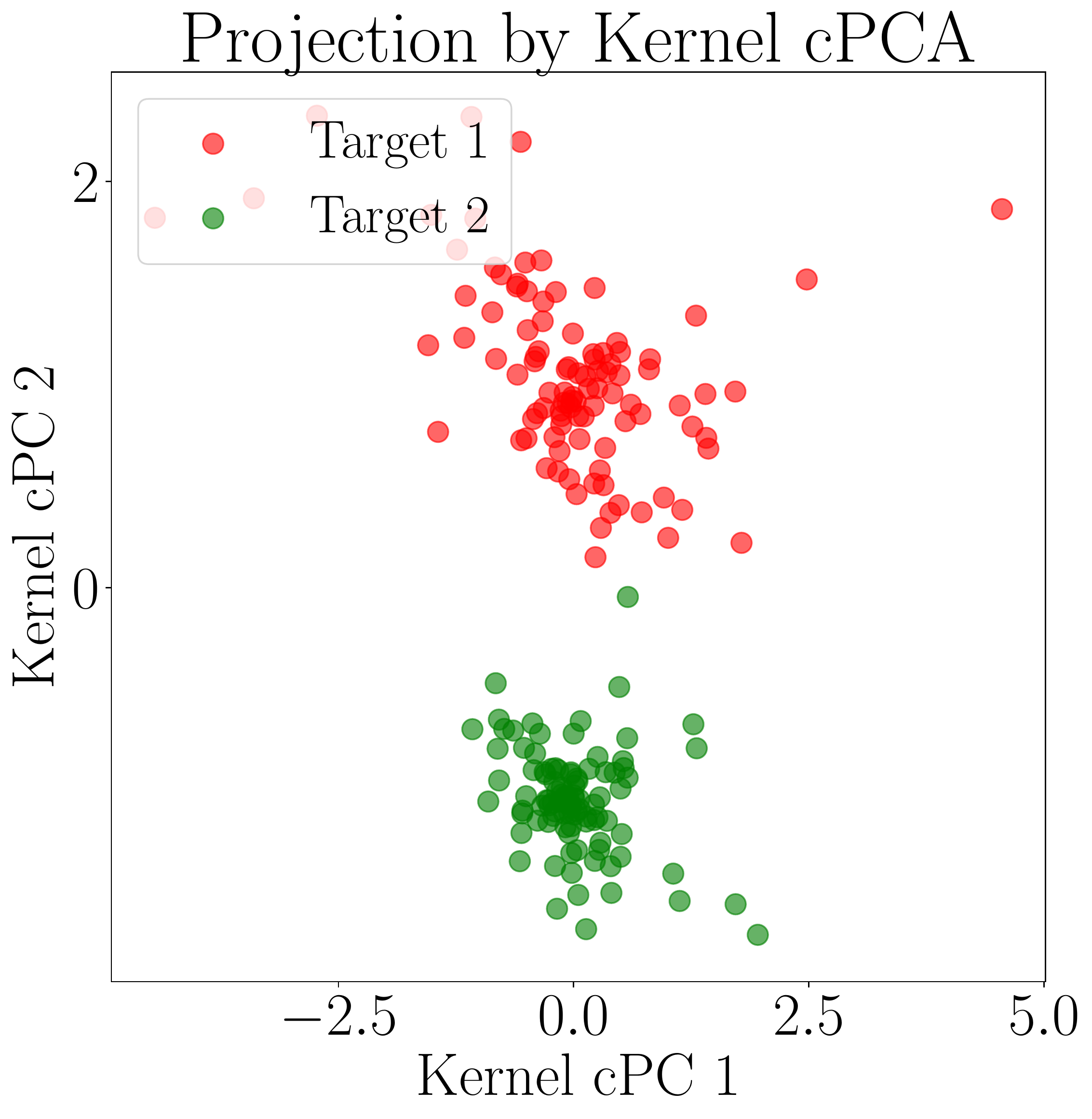}}
\caption{The results by (a) PCA, (b) cPCA, (c) kernel PCA, (d) kernel cPCA. \label{fig:toy_kernelcPCA_result}}
\end{figure}

\begin{remark} \normalfont
It is often challenging to get kernel cPCA work effectively in practice. This is because kernel cPCA is indirectly performing cPCA in the transformed feature space. However, the kernel generally induces a feature space with many correlated features, creating a large null space in the background data. Since \texttt{cPCA} does not have a penalty for directions in this null space and this null space is large, the background dataset will not be very effective at canceling out directions in the target. We plan to address this issue in the future work. 
\end{remark}

\section{Discussion}
In many data science settings, we are interested in visualizing and exploring patterns that are enriched in one dataset relative to other data. We have presented \texttt{cPCA} as a general tool for performing such contrastive exploration, and we have illustrated its usefulness in a diverse range of applications from visualizing contrastive clusters and trends to data driven normalization and denoising. The main advantages of \texttt{cPCA} is its generality and its easiness of use. Computing a particular contrast PCA takes essentially the same amount of time as computing a regular PCA. This computational efficiency enables \texttt{cPCA} to be useful for interactive data exploration, where each operation should ideally be almost immediate. Moreover any data where PCA can be usefully applied, \texttt{cPCA} can also be applied.   

For clarity of presentation, this paper focused on the setting when we want to perform contrastive analysis between one target dataset and one background dataset. In some scenarios, analyst might be interested to find the contrast between one target and \emph{multiple} background datasets. A simple approach would be to aggregate the several backgrounds into one dataset and then apply the two-population \texttt{cPCA}. More sophisticated models for simultaneous contrast across multiple datasets is an interesting direction of future work. 

The only free parameter of contrastive PCA is the contrast strength $\alpha$. In our default algorithm, we developed an automatic scheme based on clusterings of subspaces for selecting the most informative values of $\alpha$. All of the experiments performed for this paper use the automatically generated $\alpha$ values, and we believe this default will be sufficient in many applications of \texttt{cPCA}. The user may also input specific values for $\alpha$ if more fine-grained exploration is desired. 

\texttt{cPCA}, like regular PCA and other dimensionality reduction methods, does not give $p$-values or statistical significance quantifications. The patterns discovered through \texttt{cPCA} need to be validated through hypothesis testing or additional analysis using relevant domain knowledge. We have released the code for \texttt{cPCA} as a python package along with documentation and examples on Github. Links are provided in the footnotes to the abstract of this paper.

% Contrastive PCA optimizes two objectives simultaneously: (1) identify directions that have the most variation in which the target dataset, and (2) identify directions in which the target and background dataset differ in variance. We have shown in Section \ref{section:theory} that the natural set of directions that satisfy these constraints can be achieved by sweeping over various values of the contrast parameter $\alpha$. 

%At the same time, contrastive PCA is different from supervised techniques, which seek only to optimize for the second objective identified above. These techniques, such as linear and quadratic discriminant analysis \citep{}, will identify features in which two sets of data can be separated. However, the features that result will not necessarily find additional structure in the target datasets. 

%We have proposed contrastive PCA which finds subspaces which simultaneously maximizes the variance in the target dataset and minimizes variance in the background dataset. In some sense, contrastive techniques like contrastive PCA can be seen as a generalization of above methods. Small values of $\alpha$ favor the former objective and larger values of $\alpha$ favor the latter objective. As we have shown in Section \ref{section:applications}, these objectives arise naturally in a variety of applications. Our work provides an intuitive and efficient algorithm that solves these objectives, and as a result, is able to discover patterns missed by PCA on a variety of settings.

\bibliography{ref}

\begin{thebibliography}{23}
\providecommand{\natexlab}[1]{#1}
\providecommand{\url}[1]{\texttt{#1}}
\expandafter\ifx\csname urlstyle\endcsname\relax
  \providecommand{\doi}[1]{doi: #1}\else
  \providecommand{\doi}{doi: \begingroup \urlstyle{rm}\Url}\fi

\bibitem[Banos et~al.(2015)Banos, Villalonga, Garcia, Saez, Damas,
  Holgado-Terriza, Lee, Pomares, and Rojas]{Banos2015}
Banos, Oresti, Villalonga, Claudia, Garcia, Rafael, Saez, Alejandro, Damas,
  Miguel, Holgado-Terriza, Juan~A, Lee, Sungyong, Pomares, Hector, and Rojas,
  Ignacio.
\newblock Design, implementation and validation of a novel open framework for
  agile development of mobile health applications.
\newblock \emph{{BioMedical} Engineering {OnLine}}, 14\penalty0 (Suppl
  2):\penalty0 S6, 2015.

\bibitem[Bhargava et~al.(2014)Bhargava, Head, Ordoukhanian, Mercola, and
  Subramaniam]{Bhargava2014}
Bhargava, Vipul, Head, Steven~R., Ordoukhanian, Phillip, Mercola, Mark, and
  Subramaniam, Shankar.
\newblock Technical variations in low-input {RNA}-seq methodologies.
\newblock \emph{Scientific Reports}, 4\penalty0 (1), 2014.

\bibitem[Cavalli-Sforza(1998)]{cavalli1998dna}
Cavalli-Sforza, Luca~L.
\newblock The dna revolution in population genetics.
\newblock \emph{Trends in Genetics}, 14\penalty0 (2):\penalty0 60--65, 1998.

\bibitem[Cox \& Cox(2008)Cox and Cox]{cox2008multidimensional}
Cox, Michael~AA and Cox, Trevor~F.
\newblock Multidimensional scaling.
\newblock \emph{Handbook of data visualization}, pp.\  315--347, 2008.

\bibitem[du~Prel et~al.(2010)du~Prel, R{\"o}hrig, Hommel, and
  Blettner]{du2010choosing}
du~Prel, Jean-Baptist, R{\"o}hrig, Bernd, Hommel, Gerhard, and Blettner, Maria.
\newblock Choosing statistical tests: part 12 of a series on evaluation of
  scientific publications.
\newblock \emph{Deutsches {\"A}rzteblatt International}, 107\penalty0
  (19):\penalty0 343, 2010.

\bibitem[Garte(1998)]{garte1998role}
Garte, Seymour.
\newblock The role of ethnicity in cancer susceptibility gene polymorphisms:
  the example of cyp1a1.
\newblock \emph{Carcinogenesis}, 19\penalty0 (8):\penalty0 1329--1332, 1998.

\bibitem[Higuera et~al.(2015)Higuera, Gardiner, and Cios]{higuera2015self}
Higuera, Clara, Gardiner, Katheleen~J, and Cios, Krzysztof~J.
\newblock Self-organizing feature maps identify proteins critical to learning
  in a mouse model of down syndrome.
\newblock \emph{PLOS ONE}, 10\penalty0 (6):\penalty0 e0129126, 2015.

\bibitem[Hotelling(1933)]{hotelling1933analysis}
Hotelling, Harold.
\newblock Analysis of a complex of statistical variables into principal
  components.
\newblock \emph{Journal of Educational Psychology}, 24\penalty0 (6):\penalty0
  417, 1933.

\bibitem[Joliffe \& Morgan(1992)Joliffe and Morgan]{joliffe1992principal}
Joliffe, IT and Morgan, BJT.
\newblock Principal component analysis and exploratory factor analysis.
\newblock \emph{Statistical methods in medical research}, 1\penalty0
  (1):\penalty0 69--95, 1992.

\bibitem[Jolliffe(2002)]{jolliffe2002principal}
Jolliffe, Ian.
\newblock \emph{Principal Component Analysis}.
\newblock Wiley Online Library, 2002.

\bibitem[LeCun et~al.(1998)LeCun, Bottou, Bengio, and
  Haffner]{lecun1998gradient}
LeCun, Yann, Bottou, Leon, Bengio, Yoshua, and Haffner, Patrick.
\newblock Gradient-based learning applied to document recognition.
\newblock \emph{Proceedings of the IEEE}, 86\penalty0 (11):\penalty0
  2278--2324, 1998.

\bibitem[Maaten \& Hinton(2008)Maaten and Hinton]{maaten2008visualizing}
Maaten, Laurens van~der and Hinton, Geoffrey.
\newblock Visualizing data using t-sne.
\newblock \emph{Journal of Machine Learning Research}, 9\penalty0
  (Nov):\penalty0 2579--2605, 2008.

\bibitem[Miao \& Ben-Israel(1992)Miao and Ben-Israel]{Miao1992}
Miao, Jianming and Ben-Israel, Adi.
\newblock On principal angles between subspaces in rn.
\newblock \emph{Linear Algebra and its Applications}, 171:\penalty0 81--98,
  1992.

\bibitem[Moreno-Estrada et~al.(2014)Moreno-Estrada, Gignoux,
  Fern{\'a}ndez-L{\'o}pez, Zakharia, Sikora, Contreras, Acu{\~n}a-Alonzo,
  Sandoval, Eng, Romero-Hidalgo, et~al.]{moreno2014genetics}
Moreno-Estrada, Andr{\'e}s, Gignoux, Christopher~R, Fern{\'a}ndez-L{\'o}pez,
  Juan~Carlos, Zakharia, Fouad, Sikora, Martin, Contreras, Alejandra~V,
  Acu{\~n}a-Alonzo, Victor, Sandoval, Karla, Eng, Celeste, Romero-Hidalgo,
  Sandra, et~al.
\newblock The genetics of mexico recapitulates native american substructure and
  affects biomedical traits.
\newblock \emph{Science}, 344\penalty0 (6189):\penalty0 1280--1285, 2014.

\bibitem[Ng et~al.(2002)Ng, Jordan, and Weiss]{ng2002spectral}
Ng, Andrew~Y, Jordan, Michael~I, and Weiss, Yair.
\newblock On spectral clustering: Analysis and an algorithm.
\newblock In \emph{Advances in neural information processing systems}, pp.\
  849--856, 2002.

\bibitem[Novembre et~al.(2008)Novembre, Johnson, Bryc, Kutalik, Boyko, Auton,
  Indap, King, Bergmann, Nelson, Stephens, and Bustamante]{Novembre2008}
Novembre, John, Johnson, Toby, Bryc, Katarzyna, Kutalik, Zoltan, Boyko,
  Adam~R., Auton, Adam, Indap, Amit, King, Karen~S., Bergmann, Sven, Nelson,
  Matthew~R., Stephens, Matthew, and Bustamante, Carlos~D.
\newblock Genes mirror geography within europe.
\newblock \emph{Nature}, 456\penalty0 (7218):\penalty0 98--101, 2008.

\bibitem[Ringn{\'e}r(2008)]{ringner2008principal}
Ringn{\'e}r, Markus.
\newblock What is principal component analysis?
\newblock \emph{Nature Biotechnology}, 26\penalty0 (3):\penalty0 303, 2008.

\bibitem[Russakovsky et~al.(2015)Russakovsky, Deng, Su, Krause, Satheesh, Ma,
  Huang, Karpathy, Khosla, Bernstein, et~al.]{russakovsky2015imagenet}
Russakovsky, Olga, Deng, Jia, Su, Hao, Krause, Jonathan, Satheesh, Sanjeev, Ma,
  Sean, Huang, Zhiheng, Karpathy, Andrej, Khosla, Aditya, Bernstein, Michael,
  et~al.
\newblock Imagenet large scale visual recognition challenge.
\newblock \emph{International Journal of Computer Vision}, 115\penalty0
  (3):\penalty0 211--252, 2015.

\bibitem[Sch{\"o}lkopf et~al.(1997)Sch{\"o}lkopf, Smola, and
  M{\"u}ller]{scholkopf1997kernel}
Sch{\"o}lkopf, Bernhard, Smola, Alexander, and M{\"u}ller, Klaus-Robert.
\newblock Kernel principal component analysis.
\newblock In \emph{International Conference on Artificial Neural Networks},
  pp.\  583--588. Springer, 1997.

\bibitem[Silva-Zolezzi et~al.(2009)Silva-Zolezzi, Hidalgo-Miranda, Estrada-Gil,
  Fernandez-Lopez, Uribe-Figueroa, Contreras, Balam-Ortiz, del Bosque-Plata,
  Velazquez-Fernandez, Lara, Goya, Hernandez-Lemus, Davila, Barrientos, March,
  and Jimenez-Sanchez]{SilvaZolezzi2009}
Silva-Zolezzi, I., Hidalgo-Miranda, A., Estrada-Gil, J., Fernandez-Lopez,
  J.~C., Uribe-Figueroa, L., Contreras, A., Balam-Ortiz, E., del Bosque-Plata,
  L., Velazquez-Fernandez, D., Lara, C., Goya, R., Hernandez-Lemus, E., Davila,
  C., Barrientos, E., March, S., and Jimenez-Sanchez, G.
\newblock Analysis of genomic diversity in mexican mestizo populations to
  develop genomic medicine in mexico.
\newblock \emph{Proceedings of the National Academy of Sciences}, 106\penalty0
  (21):\penalty0 8611--8616, 2009.

\bibitem[van~den Berg et~al.(2006)van~den Berg, Hoefsloot, Westerhuis, Smilde,
  and van~der Werf]{vandenBerg2006}
van~den Berg, Robert~A, Hoefsloot, Huub~CJ, Westerhuis, Johan~A, Smilde, Age~K,
  and van~der Werf, Mariët~J.
\newblock \emph{{BMC} Genomics}, 7\penalty0 (1):\penalty0 142, 2006.

\bibitem[Wold et~al.(1987)Wold, Esbensen, and Geladi]{wold1987principal}
Wold, Svante, Esbensen, Kim, and Geladi, Paul.
\newblock Principal component analysis.
\newblock \emph{Chemometrics and intelligent laboratory systems}, 2\penalty0
  (1-3):\penalty0 37--52, 1987.

\bibitem[Zheng et~al.(2017)Zheng, Terry, Belgrader, Ryvkin, Bent, Wilson,
  Ziraldo, Wheeler, and P.]{Zheng2017}
Zheng, Grace X.~Y., Terry, Jessica~M., Belgrader, Phillip, Ryvkin, Paul, Bent,
  Zachary~W., Wilson, Ryan, Ziraldo, Solongo~B., Wheeler, Tobias~D., and P.,
  Geoff.
\newblock Massively parallel digital transcriptional profiling of single cells.
\newblock \emph{Nature Communications}, 8:\penalty0 14049, 2017.

\end{thebibliography}
\bibliographystyle{icml2017}

\clearpage
\onecolumn
\appendix

\begin{center}
% \textbf{\large Supplemental Materials: \\Contrastve Principal Component Analysis}
\textbf{\large Appendices}
\end{center}

\section{The cPCA Algorithm on Synthetic Data}
\label{supp:synthetic}

We create a toy dataset that provides some intuition for settings in which \texttt{cPCA} is able to resolve subgroups, and the role of the contrast parameter $\alpha$. Consider a target dataset, $\{X_i\}$, that consists of 400 data points in 30-dimensional feature space. There are four subgroups within this dataset (\textbf{\textcolor{red}{red}}, \textbf{\textcolor{blue}{blue}}, \textbf{\textcolor{olive}{yellow}}, \textbf{\textcolor{black}{black}}), each of 100 points. The first 10 features of two subgroups (\textbf{\textcolor{red}{red}}, \textbf{\textcolor{blue}{blue}}) are sampled from $N(0,1)$ while the other two subgroups (\textbf{\textcolor{black}{black}}, \textbf{\textcolor{olive}{yellow}}) are sampled from $N(6,1)$. The next 10 dimensions of the subgroups \textbf{\textcolor{red}{red}} and \textbf{\textcolor{olive}{yellow}} are sampled from $N(0,1)$ while those for the \textbf{\textcolor{black}{black}} and \textbf{\textcolor{blue}{blue}} are sampled from $N(3,1)$. The last 10 dimensions of all 400 data points are sampled from $N(0,10)$.

In this setting, classical PCA is unable to resolve the subgroups because the variance along the the last 10 dimensions is significantly larger than in any other direction, so some combination of those dimensions are selected by PCA (Panel 1, the leftmost). However, now suppose that we have a background set, ${Y_i}$ that is sampled from $N(0,3)$ along its first 10 dimensions, from $N(0,1)$ along its second 10 dimensions, and $N(0,10)$ along its final third. 

By choosing different values for $\alpha$, different subgroups can be resolved: for small values of $\alpha$, cPCA is identical to PCA, and thus the last dimensions of ${X_i}$ are selected. When $\alpha$ is increased slightly, the last dimensions are no longer selected, because they also have high variance in the background dataset. Instead, the first dimensions of ${X_i}$ are selected by contrast, allowing us to discriminate between \textbf{\textcolor{red}{red}}/\textbf{\textcolor{blue}{blue}} and the \textbf{\textcolor{black}{black}}/\textbf{\textcolor{olive}{yellow}} subgroups (Panel 2). When is $\alpha$ increased even higher, the middle dimensions of ${X_i}$ are selected by contrast as these are the dimensions that have the lowest variance in the background dataset. This allows us to discriminate between the \textbf{\textcolor{black}{black}}/\textbf{\textcolor{blue}{blue}} and \textbf{\textcolor{red}{red}}/\textbf{\textcolor{olive}{yellow}} (Panel 4). We may hypothesize that there is an intermediate value of $\alpha$ that allows us to separate all four subgroups -- indeed this is the case (Panel 3). The values of $\alpha$ (besides $\alpha=0$, which is PCA) shown below were selected automatically according to Algorithm \ref{alg:cpca2}.

\begin{figure}[h]
\centering
\includegraphics[width=0.8\textwidth]{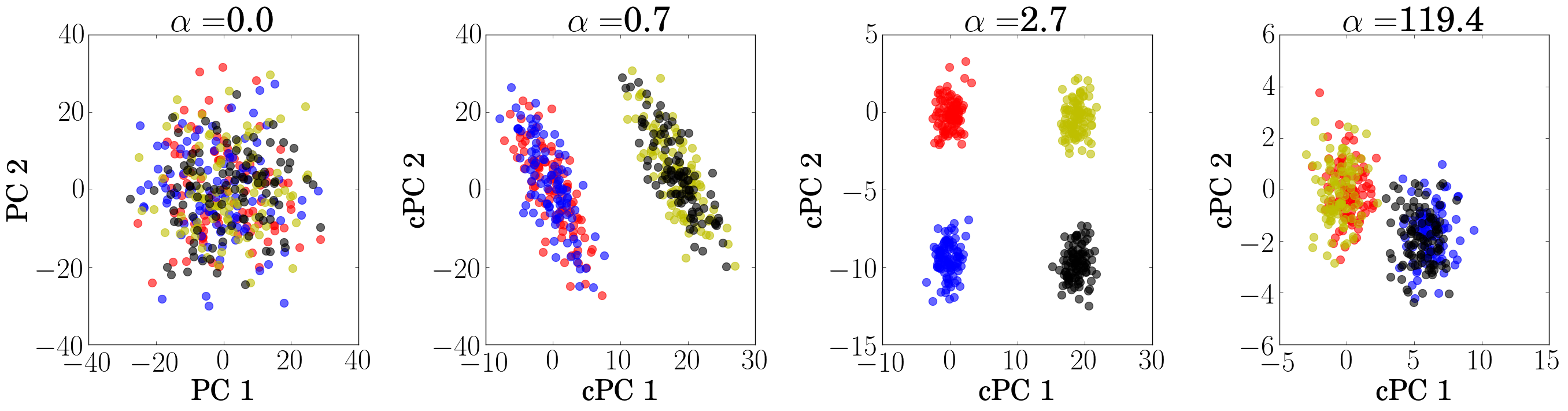}
\end{figure}

\section{Full Experimental Results of cPCA}
In this section, we include all of the figures from the cPCA analyses presented in the main body of the figure. In particular, we show the results for all three values of $\alpha$ that were automatically selected by Algorithm \ref{alg:cpca2}. We also provide the results of some additional experiments below.

\label{supp:full_results}

\subsection{Mice Protein Expression Dataset}

Here are the results for the three automatically selected values of $\alpha$, as well as PCA (corresponding to $\alpha=0$).

\begin{figure}[h]
\centering
\includegraphics[width=0.9\textwidth]{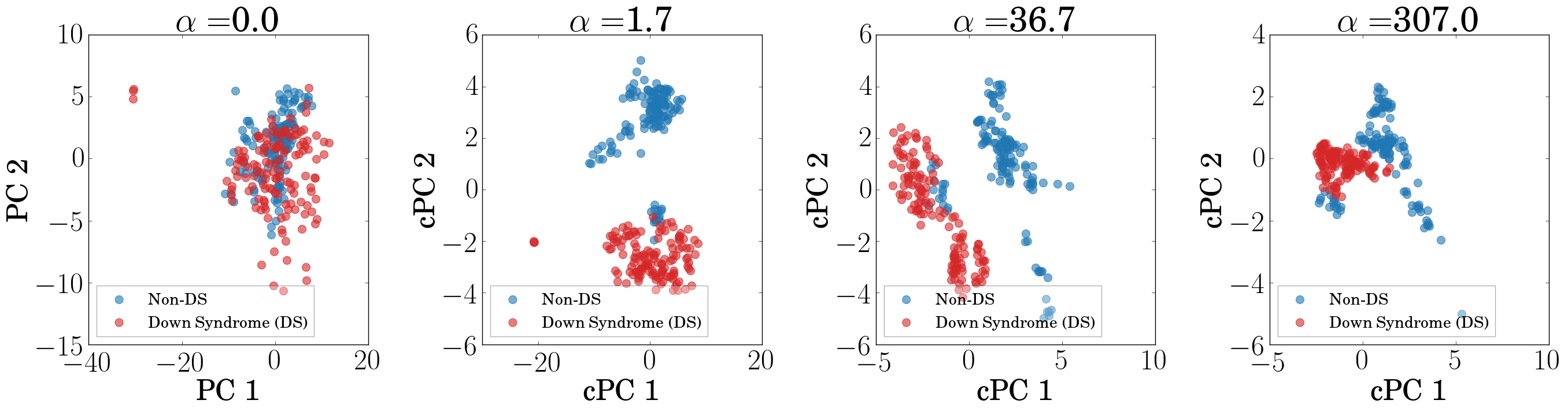}
\end{figure}

How likely is it that cPCA is discovering these clusters by chance? We can get an idea by shuffling the labels of the data and running cPCA again. A representative simulation is shown below. Because cPCA does not depend on the labels, the distribution of data points is unchanged but the labels are more randomly distributed between clusters:

\begin{figure}[H]
\centering
\includegraphics[width=0.9\textwidth]{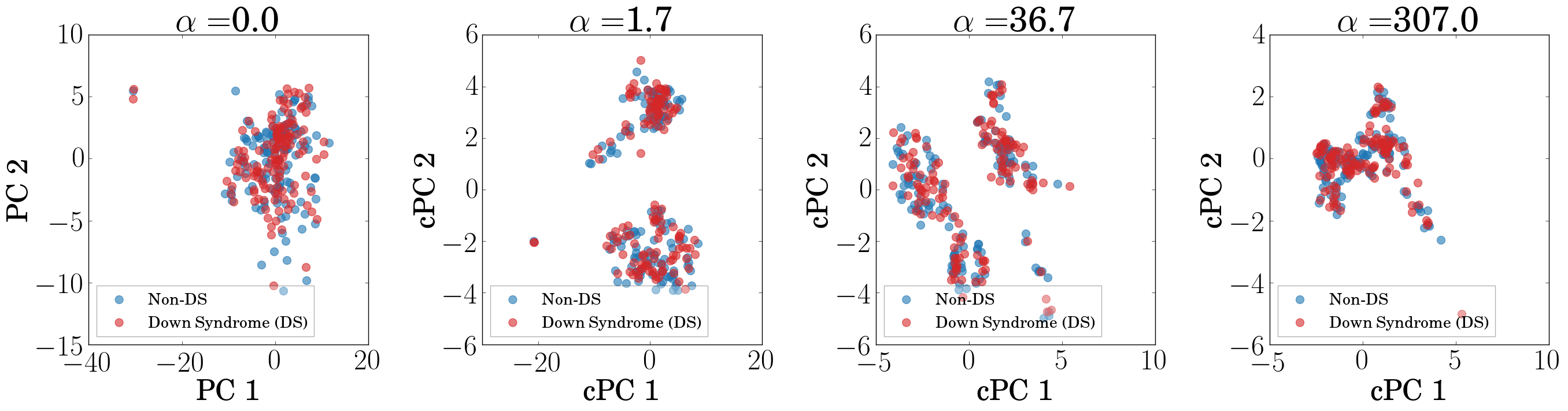}
\end{figure}

\subsection{Single Cell RNA-Seq Dataset}

For the dataset consisting of a mixture of 2 cell samples, here are the results for the three automatically selected values of $\alpha$, as well as PCA (corresponding to $\alpha=0$).

\begin{figure}[H]
\centering
\includegraphics[width=0.9\textwidth]{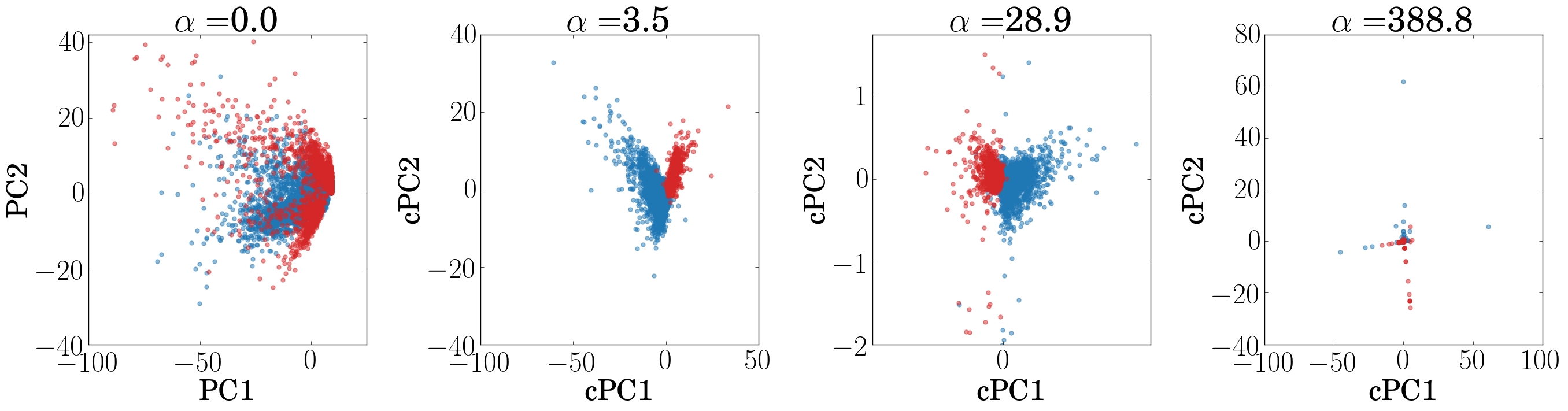}
\end{figure}

We can see more clearly the overlap between cell samples for $\alpha=28.9$ (the panel included in the main body) by plotting separately the distribution of each cell sample. See here:

\begin{figure}[H]
\centering
\subfigure{\includegraphics[width=0.23\textwidth]{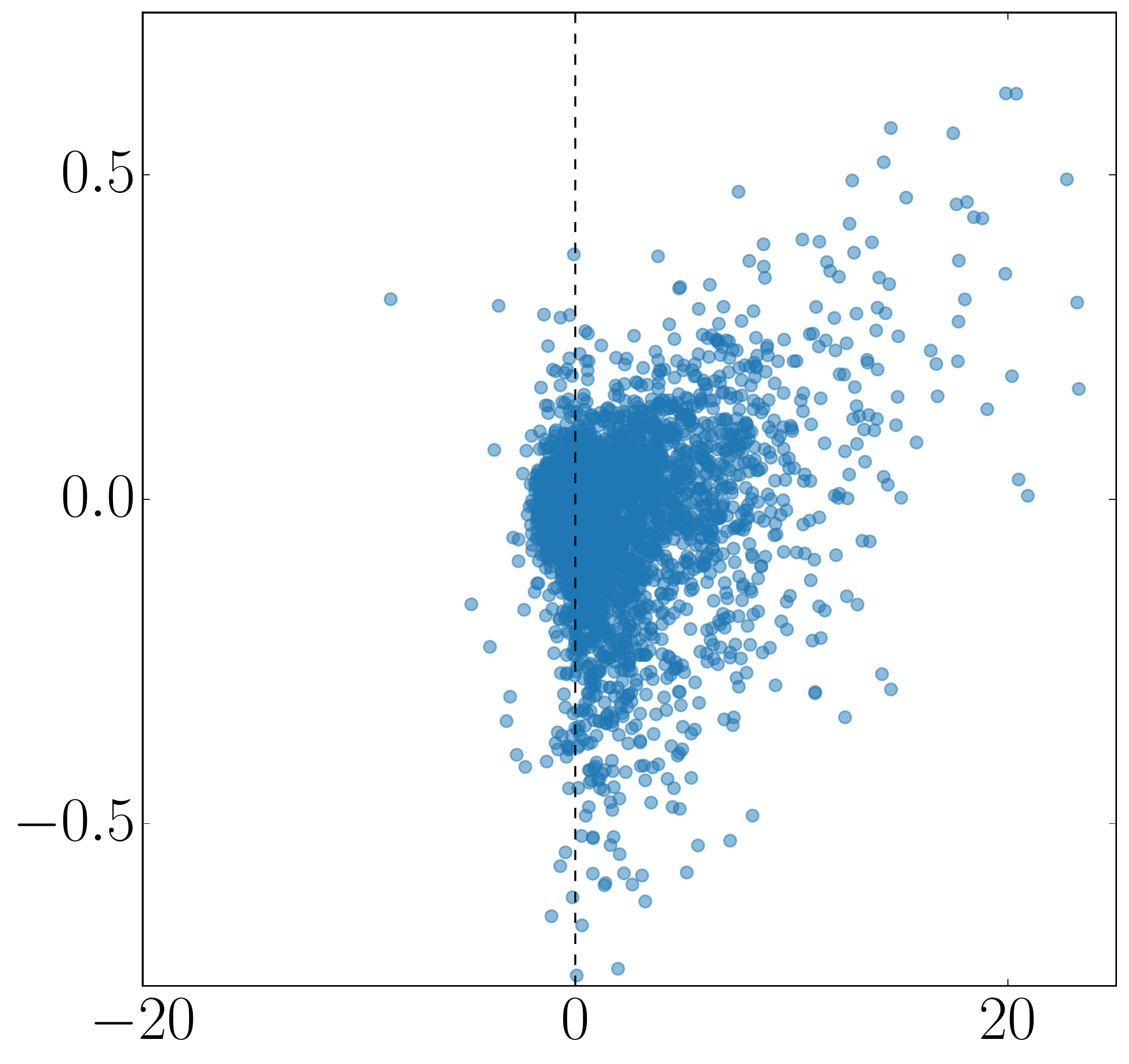}}
\subfigure{\includegraphics[width=0.23\textwidth]{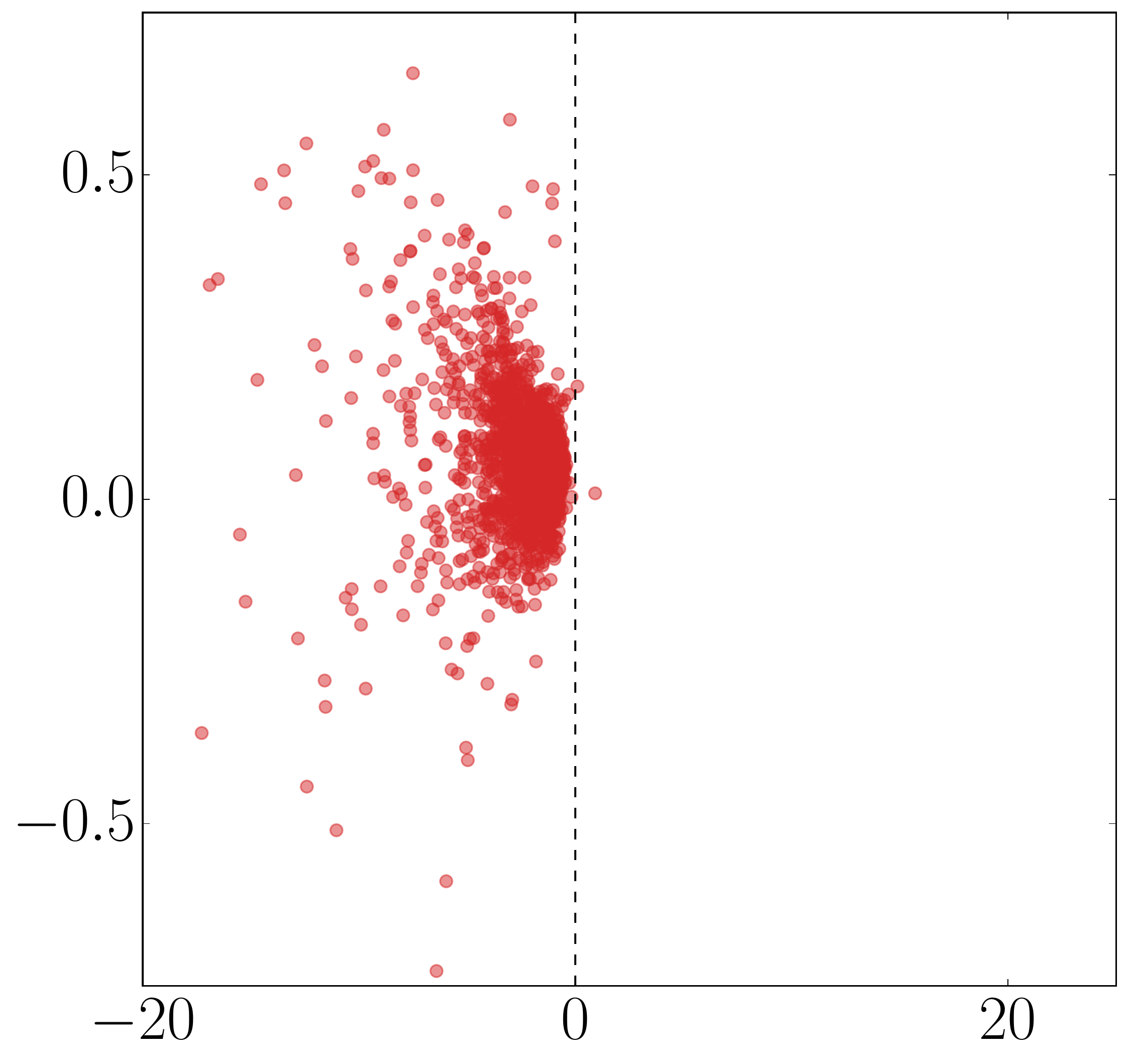}}
\end{figure}

For the dataset consisting of a mixture of 4 cell samples, here are the results for the three automatically selected values of $\alpha$, as well as PCA (corresponding to $\alpha=0$).

\begin{figure}[H]
\centering
\includegraphics[width=0.9\textwidth]{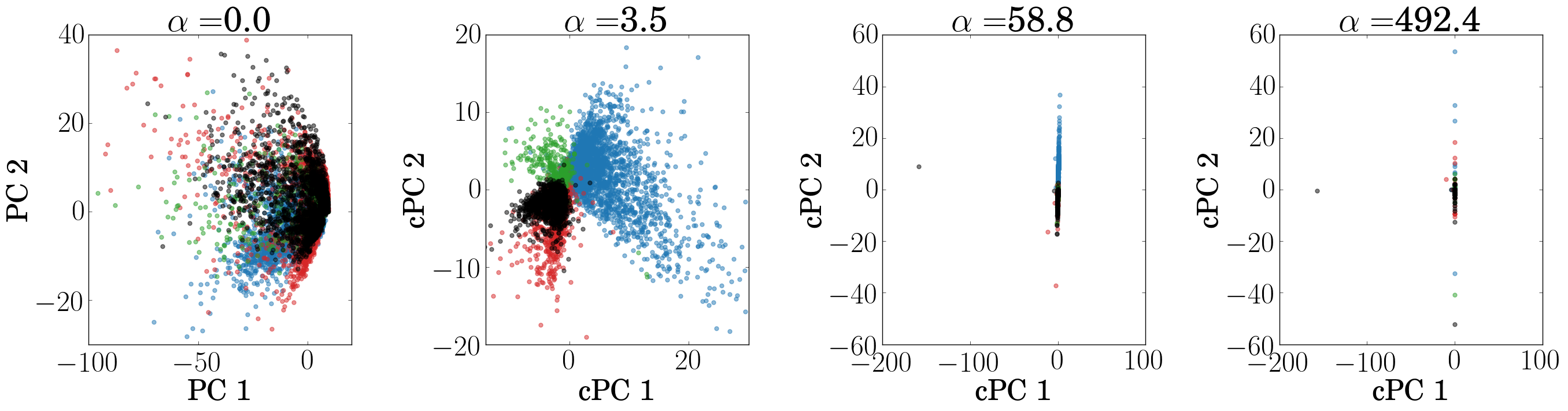}
\end{figure}

We can see more clearly the overlap between cell samples for $\alpha=3.5$ by plotting separately the distribution of each cell sample. See here:

\begin{figure}[H]
\centering
\subfigure{\includegraphics[width=0.23\textwidth]{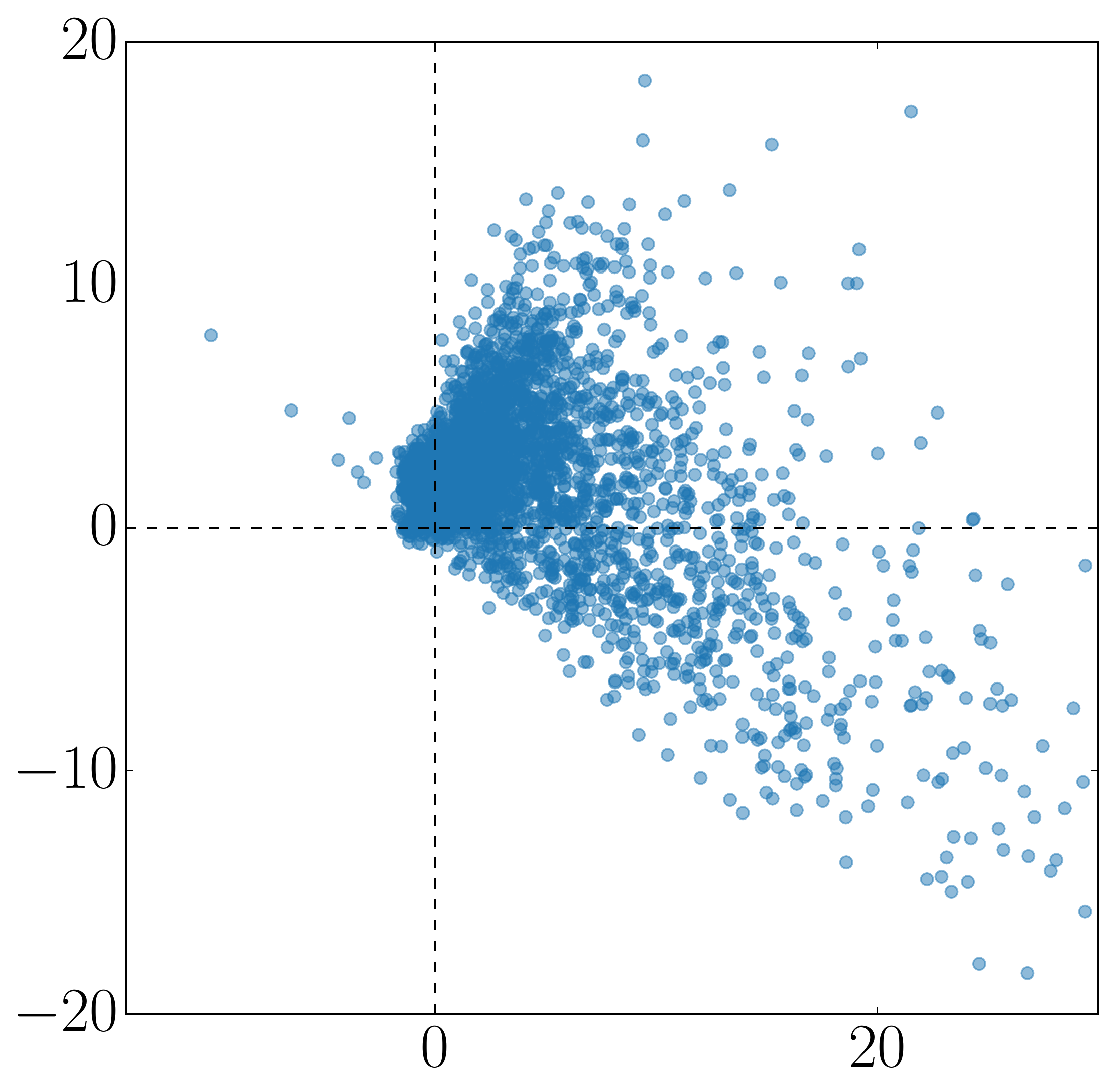}}
\subfigure{\includegraphics[width=0.23\textwidth]{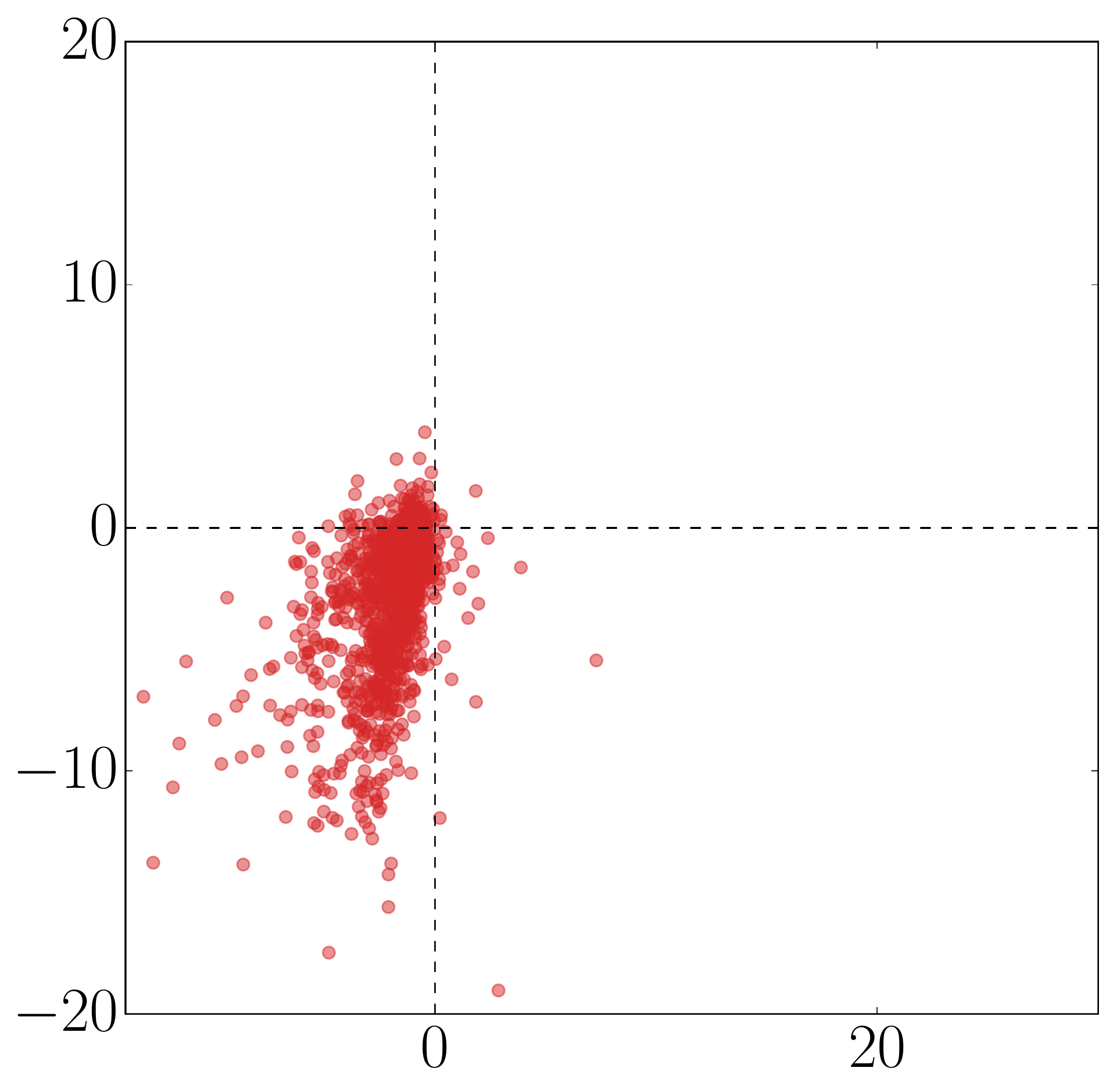}}
\subfigure{\includegraphics[width=0.23\textwidth]{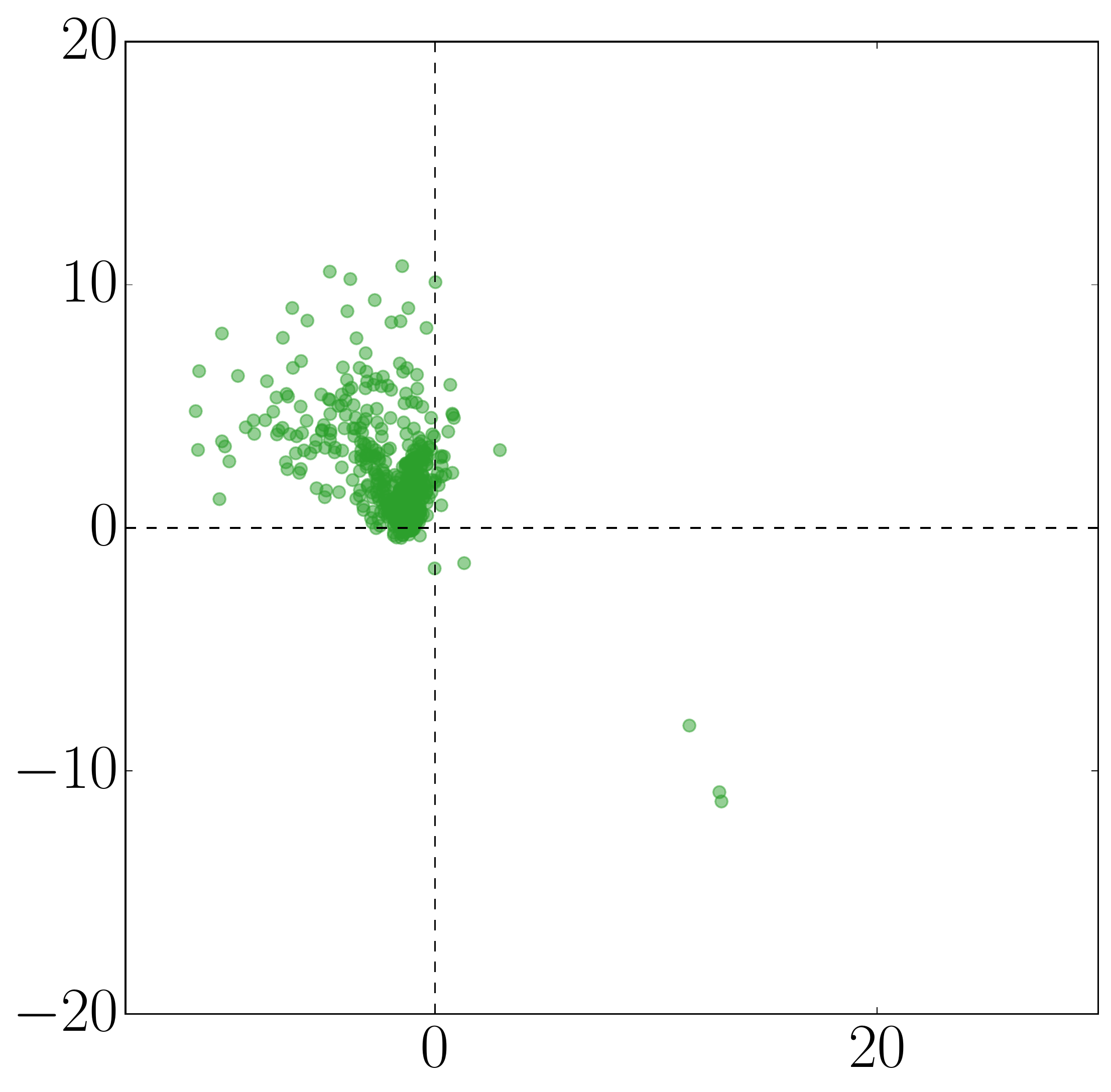}}
\subfigure{\includegraphics[width=0.23\textwidth]{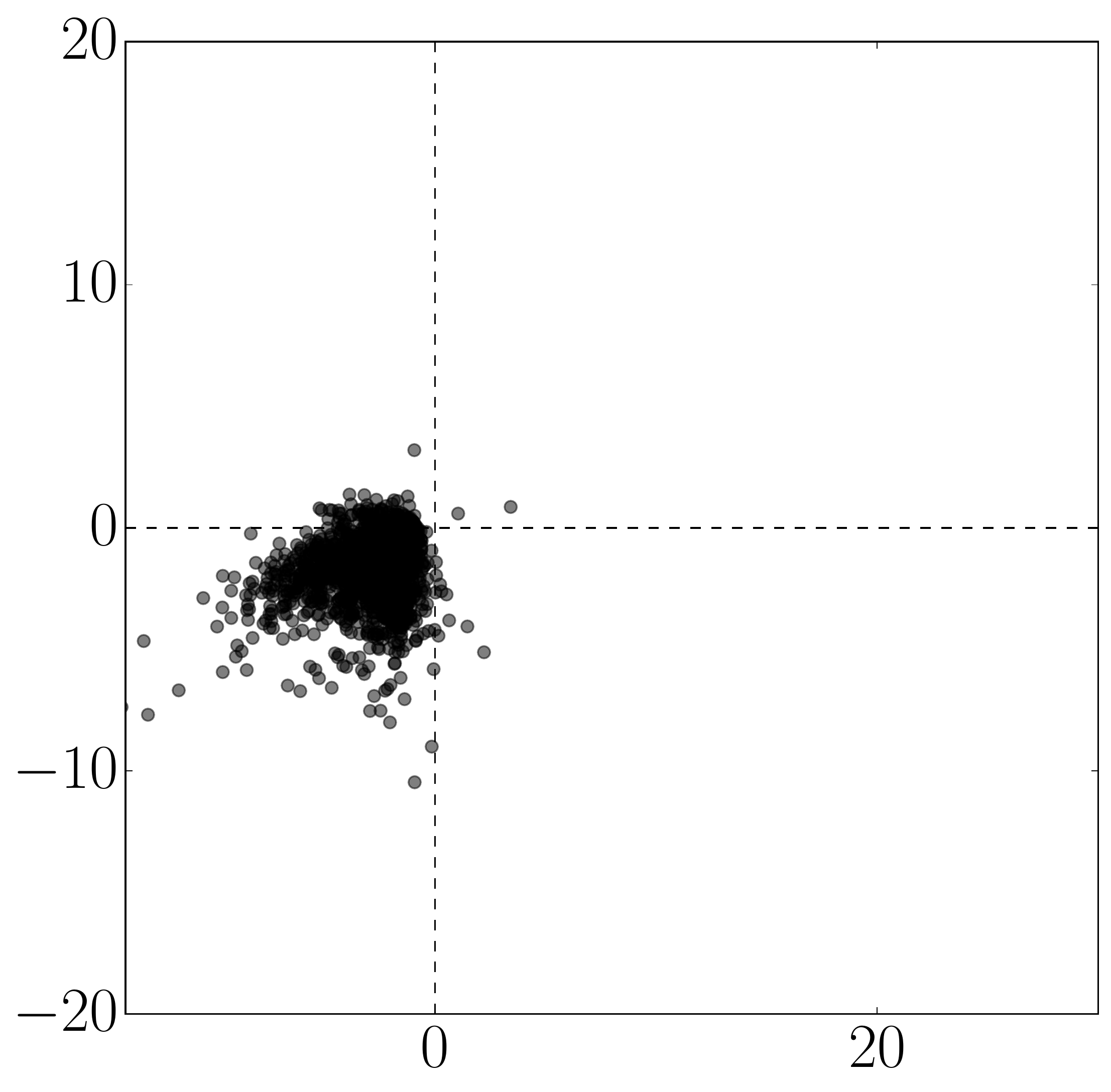}}
\end{figure}

\subsection{Mexican Ancestry Dataset}

Here are the results for the three automatically selected values of $\alpha$, as well as PCA (corresponding to $\alpha=0$).

\begin{figure}[H]
\centering
\includegraphics[width=0.9\textwidth]{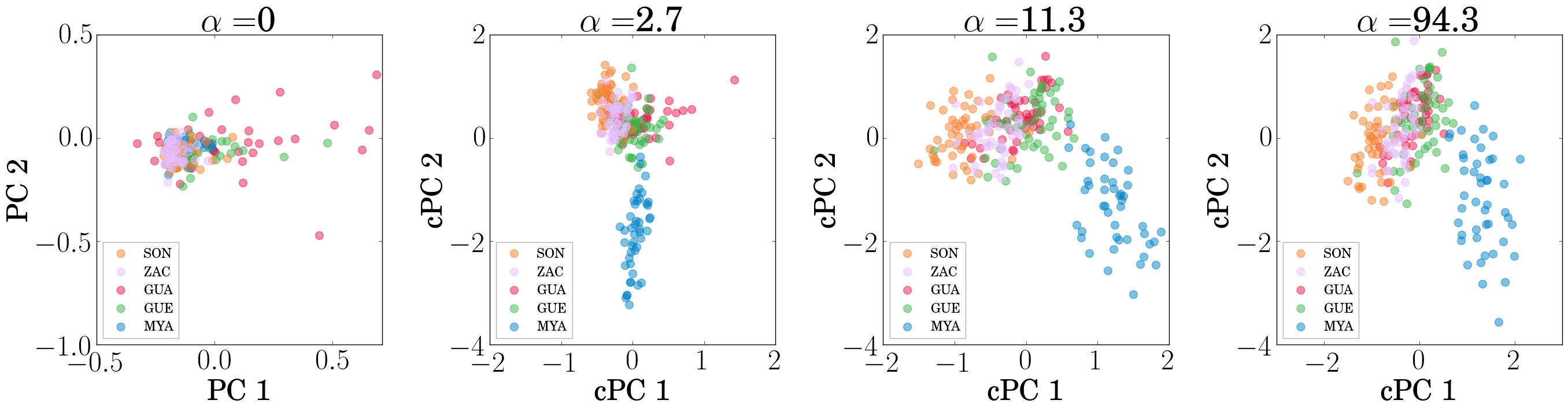}
\end{figure}

\subsection{MHealth IMU Measurements}

Here are the results for the three automatically selected values of $\alpha$, as well as PCA (corresponding to $\alpha=0$). For this example, the initial values of $\alpha$ were chosen to be 40 logarithmically spaced values from 0.1 to 1e6.

\begin{figure}[H]
\centering
\includegraphics[width=0.9\textwidth]{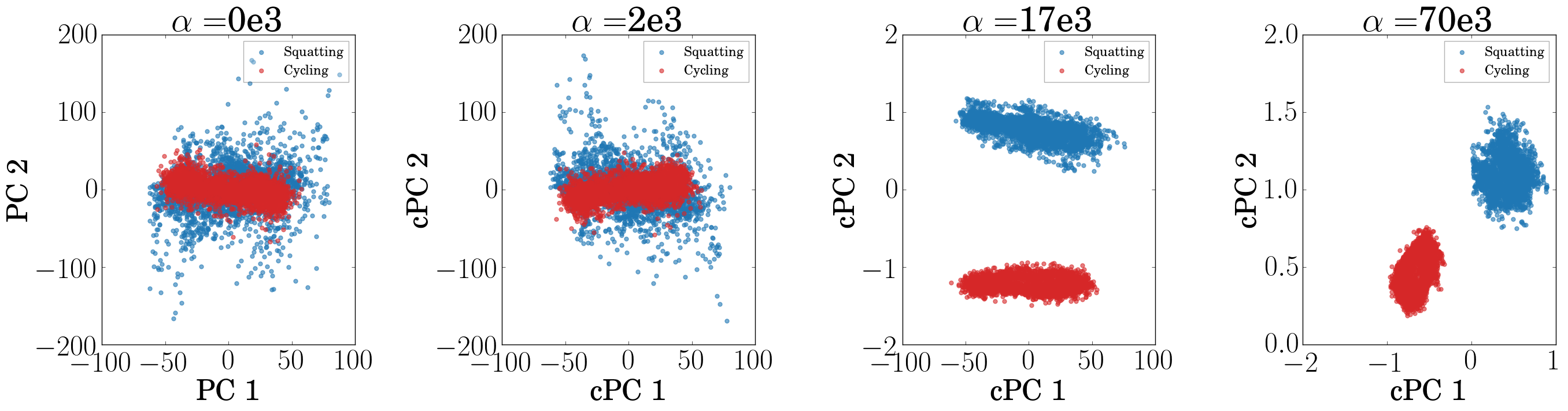}
\end{figure}

\section{Proof of Theorem \ref{thrm:opt_cPCA}}
\label{supp:proof}
\begin{proof}
Since $\mathcal{S}_\lambda$ and $\mathcal{S}_\lambda^{cPCA}$ are continuous images of $\mathcal{S}_\mathbf{v}$ and $\mathcal{S}^{cPCA}_\mathbf{v}$, it suffices to just show $\mathcal{S}^{cPCA}_\mathbf{v} = \mathcal{S}_\mathbf{v}$. 

We first show that $\mathcal{S}^{cPCA}_\mathbf{v} \subset \mathcal{S}_\mathbf{v}$.
Consider any $\mathbf{v} \in \mathcal{S}^{cPCA}_\mathbf{v}$ that is the solution of \eqref{eq:opt_obj} with alpha value $\alpha$. 
For any $\mathbf{u}\in \mathbb{R}_{unit}^d$, we have
\begin{align}
\mathbf{v}^T (C_X - \alpha C_Y) \mathbf{v} \geq \mathbf{u}^T (C_X - \alpha C_Y) \mathbf{u},
\end{align}
which can be rewritten as 
\begin{align}\label{eq:thrm_1_pf_2}
\lambda_X(\mathbf{v}) - \lambda_X (\mathbf{u})\geq \alpha ( \lambda_Y( \mathbf{v}) - \lambda_Y(\mathbf{u})). 
\end{align}
Then there are three possibilities of the relations between the variance pairs of $\mathbf{v}$ and $\mathbf{u}$: 
\begin{align*}
&1.~~ \lambda_X(\mathbf{v}) > \lambda_X (\mathbf{u}),\\
&2.~~ \lambda_X(\mathbf{v}) = \lambda_X (\mathbf{u}), ~~~~ \lambda_Y(\mathbf{v}) \leq \lambda_Y(\mathbf{u}),\\
&3.~~ \lambda_X(\mathbf{v}) < \lambda_X (\mathbf{u}), ~~~~\lambda_Y(\mathbf{v}) < \lambda_Y(\mathbf{u}).\\
\end{align*}
In all three cases, $\mathbf{u}$ can not be more contrastive than $\mathbf{v}$.
Thus $\mathbf{v} \in \mathcal{S}_\mathbf{v}$ and we can conclude that $\mathcal{S}^{cPCA}_\mathbf{v} \subset \mathcal{S}_\mathbf{v}$.

Next we show $\mathcal{S}_\mathbf{v} \subset \mathcal{S}^{cPCA}_\mathbf{v}$ by contradiction. 
Suppose there exists $\mathbf{v} \in \mathcal{S}_\mathbf{v}$ such that $\mathbf{v} \notin \mathcal{S}^{cPCA}_\mathbf{v}$. 
Since $\mathcal{S}^{cPCA}_\mathbf{v}$ and $\mathcal{S}^{cPCA}_\lambda$ are compact according to Lemma \ref{lm:compactness}, we can define
\begin{align} \label{eq:thrm_1_pf_4}
&\mathbf{v}_l = \argmax_{\mathbf{u}:~\mathbf{u}\in \mathcal{S}^{cPCA}_\mathbf{v}, \lambda_X(\mathbf{u})<\lambda_X(\mathbf{v})} \lambda_X(\mathbf{u}) \\
&\mathbf{v}_u = \argmin_{\mathbf{u}:~\mathbf{u}\in \mathcal{S}^{cPCA}_\mathbf{v}, \lambda_X(\mathbf{u})>\lambda_X(\mathbf{v})} \lambda_X(\mathbf{u}).
\end{align}
Furthermore, let $\alpha' = \frac{\lambda_X(\mathbf{v}_u)-\lambda_X(\mathbf{v}_l)}{\lambda_Y(\mathbf{v}_l)-\lambda_Y(\mathbf{v}_l)}$. 
We next show that both $\mathbf{v}_l$ and $\mathbf{v}_u$ are solutions to \eqref{eq:opt_obj} with alpha value $\alpha'$.

Since $\mathbf{v}_l$, $\mathbf{v}_u \in \mathcal{S}^{cPCA}_\mathbf{v}$, as shown previously, $\mathbf{v}_l$, $\mathbf{v}_u \in \mathcal{S}_\mathbf{v}$.
Then according to Lemma \ref{lm:cPCA_select_cond},
\begin{align*}
&\sup_{\substack{\mathbf{u}:~\mathbf{u}\in\mathcal{S}_\mathbf{v},\\\lambda_X(\mathbf{u})<\lambda_X(\mathbf{v}_l)}} \frac{\lambda_Y(\mathbf{v}_l)-\lambda_Y(\mathbf{u})}{\lambda_X(\mathbf{v}_l)-\lambda_X(\mathbf{u})} \leq \inf_{\substack{\mathbf{u}:~\mathbf{u}\in\mathcal{S}_\mathbf{v}, \\\lambda_X(\mathbf{u})>\lambda_X(\mathbf{v}_l)}} \frac{\lambda_Y(\mathbf{v}_l)-\lambda_Y(\mathbf{u})}{\lambda_X(\mathbf{v}_l)-\lambda_X(\mathbf{u})} \\&
\sup_{\substack{\mathbf{u}:~\mathbf{u}\in\mathcal{S}_\mathbf{v}, \\\lambda_X(\mathbf{u})<\lambda_X(\mathbf{v}_u)}} \frac{\lambda_Y(\mathbf{v}_u)-\lambda_Y(\mathbf{u})}{\lambda_X(\mathbf{v}_u)-\lambda_X(\mathbf{u})} \leq \inf_{\substack{\mathbf{u}:~\mathbf{u}\in\mathcal{S}_\mathbf{v}, \\\lambda_X(\mathbf{u})>\lambda)_X(\mathbf{v}_u)}} \frac{\lambda_Y(\mathbf{v}_u)-\lambda_Y(\mathbf{u})}{\lambda_X(\mathbf{v}_u)-\lambda_X(\mathbf{u})} 
\end{align*}
Then $\mathbf{v}_u$ is inside the inf term in the first equation above, and $\mathbf{v}_l$ is inside the sup term in the second equation above, both of which have the corresponding ratio $1/\alpha'$.
Then,
\begin{align}\label{eq:thrm_1_pf_3}
\sup_{\substack{\mathbf{u}:~\mathbf{u}\in\mathcal{S}_\mathbf{v},\\\lambda_X(\mathbf{u})<\lambda_X(\mathbf{v}_l)}} \frac{\lambda_Y(\mathbf{v}_l)-\lambda_Y(\mathbf{u})}{\lambda_X(\mathbf{v}_l)-\lambda_X(\mathbf{u})} \leq \frac{1}{\alpha'} \leq \inf_{\substack{\mathbf{u}:~\mathbf{u}\in\mathcal{S}_\mathbf{v},\\\lambda_X(\mathbf{u})>\lambda_X(\mathbf{v}_u)}} \frac{\lambda_Y(\mathbf{v}_u)-\lambda_Y(\mathbf{u})}{\lambda_X(\mathbf{v}_u)-\lambda_X(\mathbf{u})}
\end{align}

To show that $\mathbf{v}_l$ and $\mathbf{v}_u$ are solutions to \eqref{eq:opt_obj} with alpha value $\alpha'$, it suffices to show that $\forall \mathbf{u}\in\mathcal{S}_\mathbf{v}$, 
\begin{align*}
&\mathbf{v}_l^T (C_X -\alpha' C_Y)\mathbf{v}_l' \geq \mathbf{u}^T (C_X -\alpha' C_Y)\mathbf{u}'\\&
\mathbf{v}_u^T (C_X -\alpha' C_Y)\mathbf{v}_u' \geq \mathbf{u}^T (C_X -\alpha' C_Y)\mathbf{u}'.
\end{align*}

We consider three cases of $\mathbf{u}$.
For any $\mathbf{u}\in\mathcal{S}_\mathbf{v}$ such that $\lambda_X(\mathbf{u})<\lambda_X(\mathbf{v}_l)$, we also know $\lambda_Y(\mathbf{u})<\lambda_Y(\mathbf{v}_l)$.
According to \eqref{eq:thrm_1_pf_3},
\begin{align*}
\frac{\lambda_Y(\mathbf{v}_l)-\lambda_Y(\mathbf{u})}{\lambda_X(\mathbf{v}_l)-\lambda_X(\mathbf{u})} \leq \frac{1}{\alpha'},
\end{align*}
which is equivalent to 
\begin{align*}
\mathbf{v}_l^T (C_X -\alpha' C_Y)\mathbf{v}_l' \geq \mathbf{u}^T (C_X -\alpha' C_Y)\mathbf{u}'.
\end{align*}
Moreover, since $\frac{1}{\alpha'} = \frac{\lambda_Y(\mathbf{v}_u)-\lambda_Y(\mathbf{v}_l)}{\lambda_X(\mathbf{v}_u)-\lambda_X(\mathbf{v}_l)}$, we also have that 
\begin{align*}
\frac{\lambda_Y(\mathbf{v}_u)-\lambda_Y(\mathbf{u})}{\lambda_X(\mathbf{v}_u)-\lambda_X(\mathbf{u})} \leq \frac{1}{\alpha'},
\end{align*}
giving that 
\begin{align*}
\mathbf{v}_u^T (C_X -\alpha' C_Y)\mathbf{v}_u' \geq \mathbf{u}^T (C_X -\alpha' C_Y)\mathbf{u}'.
\end{align*}

Second, the same reasoning can be applied to the case of $\mathbf{u} \in \mathcal{S}_\mathbf{v}$ such that $\lambda_X(\mathbf{u})>\lambda_X(\mathbf{v}_u)$ 

Third, for any $\mathbf{u}\in \mathcal{S}_\mathbf{v}$ such that $\lambda_X(\mathbf{v}_l)<\lambda_X(\mathbf{u})<\lambda_X(\mathbf{v}_u)$, by definition \eqref{eq:thrm_1_pf_3}, $\mathbf{u} \notin \mathcal{S}^{cPCA}_\mathbf{v}$, and hence can not be the solution to \eqref{eq:opt_obj} with alpha value $\alpha'$. 
Therefore, $\mathbf{v}_l$ and $\mathbf{v}_u$ are solutions to \eqref{eq:opt_obj} with alpha value $\alpha'$.

Then both  $\mathbf{v}_l$ and $\mathbf{v}_u$ are eigenvectors of $C_X-\alpha' C_Y$ with the same eigenvalue. 
Then there exists $\mathbf{v}'$ in this eigenspace such that $\lambda_X(\mathbf{v}_l) < \lambda_X(\mathbf{v}') <\lambda_X(\mathbf{v}_u)$. 
We note that it is also the solution to \eqref{eq:opt_obj} with alpha value $\alpha'$ and is hence in $\mathcal{S}^{cPCA}_\mathbf{v}$. 
This contradicts the definition \eqref{eq:thrm_1_pf_4}, which completes the proof. 
\end{proof}

\section{Ancillary Lemmas}
\begin{lemma} \label{lm:compactness}
$\mathcal{S}^{cPCA}_\mathbf{v}$ and $\mathcal{S}^{cPCA}_\lambda$ are compact.
\end{lemma}
\begin{proof} (Proof of Lemma \ref{lm:compactness})
Consider any sequence of directions $\mathbf{v}_n$ in $\mathcal{S}^{cPCA}_\mathbf{v}$ that converges to $\mathbf{v}$. 
There exists a corresponding sequence of alpha's $\alpha_n$ with limit $\alpha$, where $\mathbf{v}_n$ is the solution of \eqref{eq:opt_obj} with $\alpha_n$.
Then 
\begin{align*}
& \mathbf{v}^T(C_X - \alpha C_Y) \mathbf{v} = \lim_{n\to\infty} \mathbf{v}_n^T(C_X - \alpha_n C_Y) \mathbf{v}_n \\&
= \lim_{n\to\infty} \max_{\mathbf{u} \in \mathbb{R}_{unit}^d}\mathbf{u}^T(C_X - \alpha_n C_Y) \mathbf{u} \\&
= \max_{\mathbf{u} \in \mathbb{R}_{unit}^d}\mathbf{u}^T(C_X - \alpha C_Y) \mathbf{u},
\end{align*}
giving that $\mathbf{v} \in \mathcal{S}^{cPCA}_\mathbf{v}$.
Hence $\mathcal{S}^{cPCA}_\mathbf{v}$ is compact.
Finally, being the continuous image of a compact set, $\mathcal{S}^{cPCA}_\lambda$ is also compact.
\end{proof}

\begin{lemma} \label{lm:cPCA_select_cond}
If $\mathbf{v} \in \mathcal{S}_\mathbf{v}$ and $\mathbf{v}$ is the solution to \eqref{eq:opt_obj} with value $\alpha$, then
\begin{align} \label{eq:lm_cPCA_select_cond}
\sup_{\substack{\mathbf{u}:~\mathbf{u}\in\mathcal{S}_\mathbf{v}, \\ \lambda_X(\mathbf{u})<\lambda_X(\mathbf{v})}} \frac{\lambda_Y(\mathbf{v})-\lambda_Y(\mathbf{u})}{\lambda_X(\mathbf{v})-\lambda_X(\mathbf{u})}\leq \frac{1}{\alpha} \leq \inf_{\substack{\mathbf{u}:~\mathbf{u}\in\mathcal{S}_\mathbf{v},\\\lambda_X(\mathbf{u})>\lambda_X(\mathbf{v})}} \frac{\lambda_Y(\mathbf{v})-\lambda_Y(\mathbf{u})}{\lambda_X(\mathbf{v})-\lambda_X(\mathbf{u})}. 
\end{align}
\end{lemma}

\begin{proof} (Proof of Lemma \ref{lm:cPCA_select_cond})
For any $\mathbf{u} \in \mathcal{S}_\mathbf{v}$, we have 
\begin{align*}
\mathbf{v}^T (C_X - \alpha C_Y) \mathbf{v} \geq \mathbf{u}^T (C_X - \alpha C_Y) \mathbf{u},
\end{align*}
which is equivalent to 
\begin{align}\label{eq:pf_cPCA_select_cond_1}
\lambda_X(\mathbf{v}) - \lambda_X(\mathbf{u})\geq \alpha (\lambda_Y(\mathbf{v}) - \lambda_Y(\mathbf{u})).
\end{align}
Since $\mathbf{v}, \mathbf{u} \in \mathcal{S}_\mathbf{v}$, $\lambda_X(\mathbf{v})>\lambda_X(\mathbf{u})$ implies $\lambda_Y(\mathbf{v})>\lambda_Y(\mathbf{u})$ and vice versa. 
As \eqref{eq:pf_cPCA_select_cond_1} holds for all $\mathbf{u}\in\mathcal{S}_\mathbf{v}$, this gives \eqref{eq:lm_cPCA_select_cond}.

% Other the other hand, assume \eqref{eq:lm_cPCA_select_cond} holds. 
% To show the other part of the lemma, it suffices to show that for any $\mathbf{u}\in \mathcal{S}$, 
% \begin{align} \label{eq:pf_cPCA_select_cond_1}
% \mathbf{v}^T (A - \alpha B) \mathbf{v} \geq \mathbf{u}^T (A - \alpha B) \mathbf{u}.
% \end{align}
% Note that for any $\mathbf{v},\mathbf{u} \in \mathcal{S}$, $\lambda(\mathbf{u}) > \lambda(\mathbf{v})$ indicates $\sigma(\mathbf{u}) > \sigma(\mathbf{v})$ and vice versa. 
% If $\lambda(\mathbf{u}) \neq \lambda(\mathbf{v})$ and $\sigma(\mathbf{u}) \neq \sigma(\mathbf{v})$, then \eqref{eq:lm_cPCA_select_cond} implies \eqref{eq:pf_cPCA_select_cond_1}.
% Otherwise, the fact that $\mathbf{v}\in\mathcal{S}$ implies \eqref{eq:pf_cPCA_select_cond_1}. 
% Then we shall conclude the proof. 
\end{proof}

\section{Derivation for Kernel cPCA \label{suppsec: kernelcPCA}}
Assume for the moment that the mapped data, $\Phi(X_1),\cdots,\Phi(X_n)$, $\Phi(Y_1),\cdots,\Phi(Y_m)$, is centered i.e., $\sum_{i=1}^n \Phi(X_i)=\sum_{j=1}^m \Phi(Y_j)=0$. 
The non-centered case will be considered in the end. 
The covariance matrices for the target data and background data are
\begin{align*}
\bar{A} = \frac{1}{n} \sum_{i=1}^n \Phi(X_i)\Phi(X_i)^T,~~~~\bar{B} = \frac{1}{m} \sum_{j=1}^m \Phi(Y_j)\Phi(Y_j)^T. 
\end{align*}
The contrastive components should satisfy
\begin{align}\label{eq:supp_kernelcPCA_1}
\lambda\mathbf{v} = (\bar{A}-\alpha \bar{B}) \mathbf{v},
\end{align}
where the $k$-th eigenvector corresponds to the $k$-th contrastive principal component. 
Let $N=n+m$ and define the data $Z_1,\cdots,Z_N$ as  
\begin{align*}
Z_l = \left\{ \begin{array}{cc}
X_l, & if~1\leq l \leq n \\
Y_{l-n} & otherwise
\end{array}\right..
\end{align*}

As all contrastive principal components $\mathbf{v}$ lie in the span of $\Phi(Z_1,\cdots,Z_N)$, there exists $\mathbf{a}=(a_1,\cdots,a_l) \in \mathbb{R}^N$ such that $\mathbf{v}$ can be written as  
\begin{align}\label{eq:supp_kernelcPCA_2}
\mathbf{v} = \sum_{k=1}^N a_k \Phi(Z_k).
\end{align}
Also, instead of \eqref{eq:supp_kernelcPCA_1}, we can consider the equivalent system
\begin{align}\label{eq:supp_kernelcPCA_3}
\lambda \Phi(Z_l) \cdot \mathbf{v} = \Phi(Z_l)\cdot(\bar{A}-\alpha \bar{B}) \mathbf{v},~~~~l=1,\cdots,N.
\end{align}

Substituting \eqref{eq:supp_kernelcPCA_2} into \eqref{eq:supp_kernelcPCA_3}, we have
\begin{align}\label{eq:supp_kernelcPCA_4}
\lambda \Phi(Z_l) \cdot \sum_{k=1}^N a_k \Phi(Z_k) = \Phi(Z_l)\cdot(\bar{A}-\alpha \bar{B}) \sum_{k=1}^N a_k \Phi(Z_k),~~~~for~l=1,\cdots,N.
\end{align}

Define the $N\times N$ kernel matrix $K$ by 
\begin{align}\label{eq:supp_kernelcPCA_5}
K_{ij} = \Phi(Z_i) \cdot \Phi(Z_j),
\end{align}
and further define the $N\times N$ matrices $K^A, K^B$ by 
\begin{align*}
&K^A_{ij} = \left\{ \begin{array}{cc}
K_{ij}, & if~1\leq i \leq n \\
0 & otherwise
\end{array}\right.,\\
&K^B_{ij} = \left\{ \begin{array}{cc}
0, & if~1\leq i \leq n \\
K_{ij} & otherwise
\end{array}\right.
\end{align*}

Stacking all $N$ equations together, the LHS of \eqref{eq:supp_kernelcPCA_4} is equal to $\lambda K \mathbf{a}$. 
It is also not hard to verify the RLS is equal to $K (\frac{1}{n} K^A - \frac{\alpha}{m} K^B) \mathbf{a}$.
The we can rewrite the linear system \eqref{eq:supp_kernelcPCA_4} as 
\begin{align}\label{eq:supp_kernelcPCA_6}
\lambda K \mathbf{a} = K (\frac{1}{n} K^A - \frac{\alpha}{m} K^B) \mathbf{a}.
\end{align}
To find the solution of \eqref{eq:supp_kernelcPCA_6}, we solve the eigenvalue problem 
\begin{align}\label{eq:supp_kernelcPCA_7}
\lambda \mathbf{a} = (\frac{1}{n} K^A - \frac{\alpha}{m} K^B) \mathbf{a}
\end{align}
for non-zero eigenvalues. Clearly all solutions of \eqref{eq:supp_kernelcPCA_7} do satisfy \eqref{eq:supp_kernelcPCA_6}. Also, the solutions of \eqref{eq:supp_kernelcPCA_7} and those of \eqref{eq:supp_kernelcPCA_6} differ up to a term lying in the null space of $K$. Since the projection of the data on $\mathbf{v}$ is 
\begin{align}\label{eq:supp_kernelcPCA_8}
[\Phi(Z_1)\cdot \mathbf{v},\cdots,\Phi(Z_N)\cdot \mathbf{v}]^T = K \mathbf{a},
\end{align} 
any term lying in the null space of $K$ does not affect the projected result. Hence to solve \eqref{eq:supp_kernelcPCA_6}, we can equivalently solve \eqref{eq:supp_kernelcPCA_7}. 
Finally, to impose the constraint that $\Vert \mathbf{v} \Vert=1$, we equivalently require 
\begin{align}
\mathbf{a}^T K \mathbf{a} =1. 
\end{align}

Finally, as mentioned before, the projection of the data onto the $q$-th contrastive principal component can be written as $K \mathbf{a}^{(q)}$ as \eqref{eq:supp_kernelcPCA_8}.

The centering assumption can be dropped as follows. 
Now assume that $\Phi(X_i)$ and $\Phi(Y_j)$ has some general mean $\mu_X=\frac{1}{n}\sum_{i=1}^n \Phi(X_i)$ and $\mu_Y=\frac{1}{m}\sum_{j=1}^m \Phi(Y_j)$. 
Let the non-centered kernel matrix $K$ be the same as \eqref{eq:supp_kernelcPCA_5}, and let it be partitioned into 
\begin{align}
K = \left[ \begin{array}{cc}
K_X& K_{XY}\\
K_{YX} & K_Y
\end{array}\right],
\end{align}
according to if the elements $Z_i$ and $Z_j$ belong to the target or the background data. 
Then the kernel matrix $K$ can centered as 
\begin{align}
K_{center} = \left[ \begin{array}{cc}
K_{X,center}& K_{XY,center} \\
K_{YX,center} & K_{Y,center}
\end{array}\right],
\end{align}
where 
\begin{align*}
& K_{X,center} = K_X - \mathbf{1}_n K_X - K_X \mathbf{1}_n + \mathbf{1}_n K_X \mathbf{1}_n \\
& K_{XY,center} = K_{YX} - \mathbf{1}_m K_{YX} - K_{YX} \mathbf{1}_n + \mathbf{1}_m K_{YX} \mathbf{1}_n \\
& K_{YX,center} = K_{YX} - \mathbf{1}_m K_{YX} - K_{YX} \mathbf{1}_n + \mathbf{1}_m K_{YX} \mathbf{1}_n \\
& K_{Y,center} = K_Y - \mathbf{1}_m K_Y - K_Y \mathbf{1}_m + \mathbf{1}_m K_X \mathbf{1}_m,
\end{align*}
and $\mathbf{1}_n$ and $\mathbf{1}_m$ has all elements $\frac{1}{n}$ and $\frac{1}{m}$ respectively.

\end{document}